\documentclass[twocolumn, 10pt]{article} 

\usepackage{arxiv}
\usepackage{amsmath}
\usepackage{amssymb}
\usepackage{mathtools}
\usepackage{amsthm}
\usepackage{wrapfig}
\usepackage{graphicx}
\usepackage{subcaption}
\usepackage{placeins}
\usepackage{siunitx}
\usepackage{times}
\usepackage[numbers]{natbib}
\usepackage{algorithm}
\usepackage{algorithmic}
\usepackage{booktabs}

\newtheorem{theorem}{Theorem}

\newcommand{\eat}[1]{}

\title{Learning to Guide Local Search for MPE Inference \\ in Probabilistic Graphical Models}

\author{
Brij Malhotra\\
Department of Computer Science \\
The University of Texas at Dallas\\
\texttt{brijgulsharan.malhotra@utdallas.edu}
\and
Shivvrat Arya\\
Department of Computer Science \\
New Jersey Institute of Technology\\
\texttt{shivvrat.arya@njit.edu}
\and
\hspace{2em}Tahrima Rahman\\
\hspace{2em}Department of Computer Science \\
\hspace{2em}The University of Texas at Dallas\\
\hspace{2em}\texttt{tahrima.rahman@utdallas.edu}
\and
\hspace{2em}Vibhav Gogate\\
\hspace{2em}Department of Computer Science \\
\hspace{2em}The University of Texas at Dallas\\
\hspace{2em}\texttt{vibhav.gogate@utdallas.edu}
}
\date{}

\begin{document}

\twocolumn[
  \begin{@twocolumnfalse}
    \maketitle
    \begin{abstract}
      \centering
      \begin{minipage}{0.85\textwidth} 
Most Probable Explanation (MPE) inference in Probabilistic Graphical Models (PGMs) is a fundamental yet computationally challenging problem arising in domains such as diagnosis, planning, and structured prediction. In many practical settings, the graphical model remains fixed while inference must be performed repeatedly for varying evidence patterns. Stochastic Local Search (SLS) algorithms scale to large models but rely on myopic \textit{best-improvement rule} that prioritizes immediate likelihood gains and often stagnate in poor local optima. Heuristics such as Guided Local Search (GLS+) partially alleviate this limitation by modifying the search landscape, but their guidance cannot be reused effectively across multiple inference queries on the same model. We propose a neural amortization framework for improving local search in this repeated-query regime. Exploiting the fixed graph structure, we train an attention-based network to score local moves by predicting their ability to reduce Hamming distance to a near-optimal solution. Our approach integrates seamlessly with existing local search procedures, using this signal to balance short-term likelihood gains with long-term promise during neighbor selection. We provide theoretical intuition linking distance-reducing move selection to improved convergence behavior, and empirically demonstrate consistent improvements over SLS and GLS+ on challenging high-treewidth benchmarks in the amortized inference setting.
        \vspace{1cm} 
      \end{minipage}
    \end{abstract}
  \end{@twocolumnfalse}
]

\section{Introduction}\label{sec:Intro}

Probabilistic Graphical Models (PGMs) \citep{koller-pgm} efficiently encode joint probability distributions over a large set of random variables, enabling structured reasoning under uncertainty. A fundamental inference task in these models is the Most Probable Explanation (MPE) query, where the goal is to find the most likely assignment of values to unobserved variables given observed evidence. Despite its importance in domains such as diagnosis, vision, and planning, MPE inference is NP-hard and becomes increasingly intractable as models grow in size, density, and treewidth. This scalability barrier continues to limit the applicability of PGMs in large-scale and real-world settings.

A broad range of algorithmic paradigms have been developed to address MPE inference. Exact methods, including variable elimination~\citep{dechter99} and AND/OR Branch-and-Bound (AOBB)~\citep{AOBB}, exploit problem structure to guarantee optimality. However, their time and memory requirements scale exponentially with the model's induced width, rendering them impractical for large or densely connected graphs. Although AOBB operates in an anytime fashion, its performance deteriorates as bounding quality and caching effectiveness degrade with increasing complexity. As a result, practitioners often turn to approximate inference techniques when exact solvers fail to return solutions within feasible resource budgets.

Stochastic local search (SLS) methods offer a compelling alternative for large-scale MPE inference. These methods require minimal memory, improve solutions incrementally, and parallelize naturally, making them well suited for larger graphical models. More broadly, local search has achieved strong empirical performance across large combinatorial optimization problems, including the Traveling Salesman Problem, Weighted MaxSAT, and Vehicle Routing~\citep{hoos-sls}. A typical local search algorithm alternates between \emph{intensification}, which exploits the local neighborhood of the current solution, and \emph{diversification}, which promotes exploration to escape local optima. Intensification strategies include best-improvement, first-improvement, and steepest descent~\citep{Aloise2025}, while diversification is implemented through mechanisms such as controlled randomness, adaptive memory, or neighborhood perturbations schemes based on metaheuristics including Simulated Annealing~\citep{Russell-AIMA}, Tabu Search~\citep{Glover1993}, Guided Local Search (GLS)~\citep{Tsang1997}, Variable Neighborhood Search~\citep{Brimberg2023}, and Large Neighborhood Search~\citep{Shaw98}. Achieving an effective balance between the two operations is essential for robust performance.

In the context of MPE inference, the \emph{best-improvement} heuristic is a standard intensification strategy for SLS. At each step, it selects the neighboring assignment that maximally increases log-likelihood. While effective in the immediate neighborhood, this heuristic is inherently myopic and frequently becomes trapped in poor local optima, especially in high-dimensional models with high treewidth~\citep{pmlr-vR2-kask99a}. Penalty-based extensions such as GLS~\citep{Park-gls, gls+} attempt to mitigate this issue by discouraging revisits to previously explored states. However, these penalties are non-persistent and query-specific, preventing the reuse of accumulated structural knowledge across multiple MPE queries on the same fixed graphical model. This exposes a fundamental limitation of existing GLS approaches: they lack a mechanism to amortize long-horizon experience.

To address both the short-sightedness of SLS and the non-transferability of GLS-style penalties, we introduce an \emph{amortized lookahead} strategy for neighbor selection~\citep{Meignan2015}. In the idealized setting where the optimal solution is known, the best intensification move from any state is one that reduces the Hamming distance of a state to the known optimal solution, as such moves reach the optimum in the minimum number of steps. However, this oracle information is unavailable at inference time, since the optimal assignment is precisely the quantity being sought. Our key observation is that this guidance can be approximated from solved instances. We therefore learn a neural model from (near-)optimal solutions that predicts which local moves are likely to reduce Hamming distance, enabling amortized inference that guides search toward globally promising regions across queries.

Concretely, we first obtain near-optimal solutions for a small set of training queries using an anytime MPE solver~\citep{marinescu2003systematic}. We then train an attention-based neural network to score candidate neighbors of a given assignment by estimating their likelihood of reducing Hamming distance to a high-quality solution. At inference time, these learned scores are combined with the standard best-improvement heuristic, enabling the search to balance immediate likelihood gains against long-term progress. This integration yields a reusable, \textit{amortized lookahead mechanism} for local search~\citep{Andrea-LS-look-97}. By distilling long-horizon search behavior into a learned policy, our approach improves exploration while sharing guidance across all future MPE queries on the same graphical model.

\paragraph{Contributions.} This paper makes the following contributions:

\begin{itemize}
    \item We introduce an amortized lookahead framework that formulates neighbor selection in local search for MPE inference as a supervised learning problem, explicitly balancing short-term likelihood improvement with long-term convergence.
    \item We propose a principled data generation strategy that leverages anytime MPE solvers and Hamming distance–based supervision to label neighboring assignments by their likelihood of reducing distance to a near-optimal solution.
    \item We demonstrate that augmenting SLS and GLS+ with our learned lookahead strategy consistently improves solution quality and efficiency on large and challenging graphical models.
\end{itemize}

\section{Background}

\paragraph{Notation:} We denote random variables by uppercase letters (e.g., \( X \)), and their assignments by corresponding lowercase letters (e.g., \( x \)). 
Bold uppercase letters (e.g.,~$\mathbf{X}$) denote sets of variables, while bold lowercase letters (e.g.,~$\mathbf{x}$) denote corresponding value assignments. Given a full assignment~$\mathbf{x}$ to~$\mathbf{X}$ and a subset~$\mathbf{Y} \subseteq \mathbf{X}$, we denote by~$\mathbf{x}_{\mathbf{Y}}$ the projection of~$\mathbf{x}$ onto~$\mathbf{Y}$.




A \textbf{probabilistic graphical model} (PGM) is a triple $\mathcal{M} = \langle \mathbf{X}, \mathbf{F}, G \rangle$, where (1) $\mathbf{X} = \{X_1, \dots, X_n\}$ is a set of discrete random variables; (2) each $X_i$ takes values from a finite domain $D_{X_i}$; (3) $\mathbf{F} = \{f_1, \dots, f_m\}$ is a collection of log-potential functions such that each $f_i$ is defined over a subset $\mathbf{S}(f_i) \subseteq \mathbf{X}$ (its \emph{scope}); and (4) $G = (\mathcal{V}, \mathcal{E})$ is the undirected primal graph with $\mathcal{V} = \mathbf{X}$ and an edge $(X_a, X_b) \in \mathcal{E}$ whenever $X_a$ and $X_b$ co-occur in the scope of some $f_i \in \mathbf{F}$. The model induces the joint distribution
\[
P_{\mathcal{M}}(\mathbf{x}) \propto \exp\!\left( \sum_{f \in \mathbf{F}} f(\mathbf{x}_{\mathbf{S}(f)}) \right).
\]

We focus on the \textbf{most probable explanation} (MPE) task: given observed evidence variables $\mathbf{E} \subset \mathbf{X}$ with assignment $\mathbf{e}$, find the most likely assignment to the query variables $\mathbf{Q} = \mathbf{X} \setminus \mathbf{E}$. Formally,
\begin{align}
\text{MPE}(\mathbf{Q}, \mathbf{e}) 
&= \arg\max_{\mathbf{q}} \;\text{ln}\Big(P_{\mathcal{M}}(\mathbf{q} \mid \mathbf{e})\Big) \nonumber \\
&= \arg\max_{\mathbf{q}} \sum_{f \in \mathbf{F}} f\big((\mathbf{q}, \mathbf{e})_{\mathbf{S}(f)}\big)
\end{align}
This problem is NP-hard in general and remains intractable for many expressive model classes~\citep{cooper_1990_complexityprobabilistic, park&darwiche04, decamposNewComplexityResultsMAPBayesianNetworks, conaty17}.

Exact algorithms for MPE include \emph{bucket elimination}~\citep{dechter99}, which performs variable elimination via local reparameterization; \emph{AND/OR Branch-and-Bound} framework~\citep{AOBB}, a search-based approach that exploits graphical structure through AND/OR search spaces and soft arc consistency methods, implemented in solvers such as \textsc{DAOOPT}~\citep{marinescu2010daoopt} \& \textsc{Toulbar2}~\citep{givry:hal-04021879} respectively that operate in an \emph{anytime} manner, progressively refining feasible solutions and bounds and providing certificates of optimality upon termination.

In contrast, local search methods operate over complete assignments through iterative refinement. Given a current assignment \( \mathbf{x} \), candidate updates are generated by exploring its \emph{neighborhood}. We focus on the standard \emph{1-flip neighborhood}, defined as $\mathcal{N}(\mathbf{x}) = \{\mathbf{x}' : \mathbf{x}' \text{ differs from } \mathbf{x} \text{ in exactly one query variable } Q_i \in \mathbf{Q}\}$, where the selected variable \( Q_i \) is reassigned to a different value in its domain.

A variety of local search strategies have been proposed for MPE inference. Penalty-based metaheuristics, such as Guided Local Search (GLS)~\citep{Park-gls} and its enhanced variant GLS+~\citep{gls+}, combine greedy or stochastic best-improvement moves with adaptive penalties to escape local optima. Structure-aware heuristics further exploit graphical topology: Subtree-Tree Local Search (STLS)~\citep{Milchgrub2014} searches over cutset variables while solving the remaining tree-structured subproblem exactly, while more recently a variable neighborhood search framework guided by tree decompositions that progressively expands neighborhood complexity to improve solution quality was proposed~\citep{Ouali2020}.

Heuristics are central to the effectiveness of local search methods in combinatorial optimization, where large and structured search spaces make exhaustive exploration infeasible. Across a wide range of satisfaction and optimization problems, carefully designed neighbor-selection strategies play a crucial role in balancing greedy improvement with exploration of promising regions, directly influencing convergence speed and solution quality~\citep{bg-ls-maxsat, Li2005, Balint2012, Cai2021}.

Recently, learning-based approaches have emerged as a powerful mechanism for augmenting local search heuristics by amortizing guidance from past search trajectories to inform and improve future optimization decisions. Prior work has explored neural and adaptive strategies for this purpose, including regret-based penalties for local search in the traveling salesperson problem~\citep{sui-gls-tsp}, learned neighborhood selection policies within Large Neighborhood Search frameworks for MaxSAT~\citep{Hickey2022}, and bandit-based methods that adaptively select constraints when search stalls~\citep{Zheng2022}. Together, these works show a growing trend of learning to amortize search, improving convergence and solution quality.


The challenges addressed by these approaches closely parallel those encountered in MPE inference for PGMs, where high-dimensional structured spaces and complex dependencies give rise to numerous poor local optima. As in other combinatorial settings, effective MPE inference relies critically on neighbor-selection heuristics that balance immediate objective improvement with exploration of promising long-term trajectories. Motivated by this connection, we propose a neural lookahead framework that evaluates candidate moves not only by their instantaneous likelihood gain but also by their estimated long-term potential, addressing a core limitation of greedy local search strategies.

\section{Learning to Guide Local Search}
\label{sec:method}


A widely used greedy approach in local search is the \emph{best-improvement} heuristic, which selects the neighbor $\mathbf{x}' \in \mathcal{N}(\mathbf{x})$ of $\mathbf{x}$ that maximizes the immediate gain in the objective. More formally, given $\mathbf{x}$, define the log-potential function $F$ as
\[
F(\mathbf{x}) = \sum_{f \in F} f(\mathbf{x}_{S(f)}),
\]
and the log-likelihood gain of moving to a neighbor $\mathbf{x}'$ as
\[
\mathcal{S}_{\mathrm{LL}}(\mathbf{x}'\mid \mathbf{x}) = F(\mathbf{x}') - F(\mathbf{x}).
\]
The best-improvement heuristic selects a neighbor $\mathbf{x}' \in \mathcal{N}(\mathbf{x})$ that maximizes $\mathcal{S}_{\mathrm{LL}}(\mathbf{x}'\mid \mathbf{x})$.
This rule is attractive because $\mathcal{S}_{\mathrm{LL}}$ is exact and always available at test time.
However, it is inherently myopic: repeated greedy updates often trap the search in poor local optima and prevent exploration of more promising regions.

\paragraph{A distance-based view of progress.}
To reason about sustained progress, it is useful to temporarily separate objective improvement from trajectory quality and ask a simpler question: does the search move toward an optimal MPE solution?
Let $\mathbf{x}^\ast$ denote an optimal MPE assignment.
To measure the distance between two assignments $\mathbf{x}$ and $\mathbf{z}$, we use the Hamming distance $d_H(\mathbf{x},\mathbf{z})$, defined as the number of query variables $Q_i \in \mathbf{Q}$ on which they assign different values (the evidence variables are fixed and therefore identical in both assignments).



Suppose we had access to an oracle that outputs $d_H(\mathbf{x}, \mathbf{x}^\ast)$. Then, given a starting assignment $\mathbf{x}$, a 1-flip local search could always select a neighbor $\mathbf{x}' \in \mathcal{N}(\mathbf{x})$ that decreases the distance by 1, converging in exactly $d_H(\mathbf{x}, \mathbf{x}^\ast)$ steps, which is optimal. However, strict optimality is not required for convergence. As we show in the following theorem, even a noisy oracle which selects a \textit{distance-reducing move} with probability $\alpha > 0.5$ ensures that the search converges in bounded expected time:

\begin{theorem}[Convergence]
\label{thm:drift}
Let $\{\mathbf{x}_t\}_{t \geq 0}$ be the sequence of assignments generated by 1-flip local search, with $\mathbf{x}_t$ denoting the state at step $t$.
Assume the process is absorbing at the target assignment, i.e., if $\mathbf{x}_t = \mathbf{x}^\ast$ then $\mathbf{x}_{t+1} = \mathbf{x}^\ast$.
Define
$h_t := d_H(\mathbf{x}_t,\mathbf{x}^\ast)$
as the Hamming distance to $\mathbf{x}^\ast$, and let
$\tau^\ast := \inf\{t \ge 0 : h_t = 0\}$
be the hitting time of $\mathbf{x}^\ast$.
If at every non-terminal state the executed move reduces $h_t$ with probability $\alpha > 1/2$, then
\[
\mathbb{E}[h_{t+1} \mid \mathbf{x}_t] \le h_t - (2\alpha-1),
\quad \text{and} \quad
\mathbb{E}[\tau^\ast] \le \tfrac{h_0}{2\alpha-1}.
\]
\end{theorem}
Theorem~\ref{thm:drift} states that for reliable progress, the search must select a distance-reducing move more often than not ($\alpha > 0.5$).
Since our method selects the neighbor with the highest predicted score, $\alpha$ is a top-1 ranking statistic representing the probability that the selected move reduces the Hamming distance.
Defining the set of distance-reducing neighbors as $\mathcal{N}_{\downarrow}(\mathbf{x}) := \{\mathbf{z} \in \mathcal{N}(\mathbf{x}) : d_H(\mathbf{z}, \mathbf{x}^\ast) = d_H(\mathbf{x}, \mathbf{x}^\ast) - 1\}$, we formalize $\alpha$ as:
\[
\alpha
=
\Pr\!\left[\mathbf{x}_{t+1} \in \mathcal{N}_{\downarrow}(\mathbf{x}_t) \mid d_H(\mathbf{x}_t, \mathbf{x}^\ast) > 0\right]
\]
Thus, strict score calibration is unnecessary; the learned policy need only rank distance-reducing moves above non-reducing ones often enough to ensure $\alpha > 1/2$ along the trajectory.

\paragraph{Learning a surrogate oracle.}
In practice, since the optimal assignment $\mathbf{x}^\ast$ is typically unavailable, we supervise the model using a reference assignment $\hat{\mathbf{x}}$ produced by an anytime solver.
We label a neighbor $\mathbf{x}' \in \mathcal{N}(\mathbf{x})$ as positive ($y=1$) if it reduces the Hamming distance to $\hat{\mathbf{x}}$ (i.e., $d_H(\mathbf{x}', \hat{\mathbf{x}}) < d_H(\mathbf{x}, \hat{\mathbf{x}})$), and negative ($y=0$) otherwise.
We train a neural network to output a score $\hat{p}_{\downarrow}(\mathbf{x}, \mathbf{x}')$ that estimates the probability $\Pr(y=1 \mid \mathbf{x}, \mathbf{x}')$.
At test time, the search executes the top-ranked move:
\[
\mathbf{x}_{t+1} = \arg\max_{\mathbf{z} \in \mathcal{N}(\mathbf{x}_t)} \hat{p}_{\downarrow}(\mathbf{x}_t, \mathbf{z}).
\]

Section~\ref{subsec:data} details the generation of $\hat{\mathbf{x}}$ and the collection of training states to match the inference distribution.

\subsection{Data Generation for Neighbor Scoring}\label{subsec:data}

\begin{algorithm}[tb]
\caption{Data Collection}
\label{algo:data}
\begin{algorithmic}
\footnotesize
   \STATE {\bfseries Input:} A PGM $\mathcal{M}$, query ratio $qr$, anytime solver budget $B$, local search step limit $stl$, number of samples $N$
   \STATE {\bfseries Output:} Dataset $DB$
   \STATE

   \STATE \textbf{function} \textsc{SolveMPEAnytime}($\mathcal{M}, \mathbf{e}, B$)
   \STATE \quad Run an anytime MPE solver on $\mathcal{M}$ with evidence $\mathbf{e}$ 
   \STATE \quad and time budget $B$
   \STATE \quad \textbf{return} high-quality assignment $\hat{\mathbf{x}}$
   \STATE \textbf{end function}

   \STATE
   \STATE Initialize $DB \gets \emptyset$
   \FOR{$i=1$ {\bfseries to} $N$}
       \STATE Sample $\mathbf{x} \sim P_{\mathcal{M}}$
       \STATE $\mathbf{Q} \gets$ random subset of $\mathbf{X}$ of size $qr \cdot |\mathbf{X}|$
       \STATE $\mathbf{E} \gets \mathbf{X} \setminus \mathbf{Q}$
       \STATE $\mathbf{e} \gets \mathbf{x}_{\mathbf{E}}$

       \STATE $\hat{\mathbf{x}} \gets \textsc{SolveMPEAnytime}(\mathcal{M}, \mathbf{e}, B)$

       \STATE Run local search with evidence $\mathbf{e}$ and step limit $stl$ to collect states $\mathbf{S}$ and neighbors $\mathcal{N}(\mathbf{x})$ for each $\mathbf{x} \in \mathbf{S}$

       \FORALL{$\mathbf{x} \in \mathbf{S}$}
           \STATE $\mathcal{D}_{x} \gets \text{List of pairs}$
           \FORALL{$\mathbf{x'} \in \mathcal{N}(\mathbf{x})$}
               \STATE $label \gets \mathbb{I}[d_{H}(x^{\prime},\hat{x})<d_{H}(x,\hat{x})]$
               \STATE $\mathcal{D}_{x} \gets \mathcal{D}_{x}\cup\{(x', label)\}$
           \ENDFOR
           \STATE $\text{record} \gets \{\mathbf{e}, \mathbf{x}, \mathcal{D}_{x}\}$
           \STATE $DB \gets DB \cup \{\text{record}\}$
       \ENDFOR
   \ENDFOR
   \STATE \textbf{return} $DB$
\end{algorithmic}
\end{algorithm}


Solving the MPE problem exactly is computationally intractable for large, high-treewidth networks, making the true optimal assignment $\mathbf{x}^\ast$ unavailable in practice. To overcome this, we supervise learning using \emph{high-quality approximate solutions} $\hat{\mathbf{x}}$ obtained from anytime MPE solvers.

\begin{figure*}[thbp!]
    \centering
    \includegraphics[width=\linewidth]{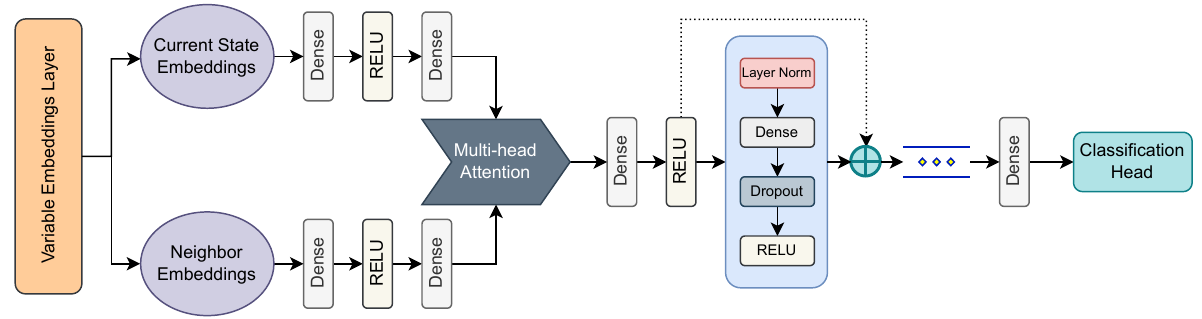}
    \caption{Attention-based neural architecture for scoring 1-flip neighbors by their predicted utility in reducing the Hamming distance to a high-quality assignment.}
    \label{fig:architecture}
\end{figure*}

Algorithm~\ref{algo:data} describes the data collection procedure. Given a probabilistic graphical model \(\mathcal{M}\) and a query ratio \(qr \in (0,1)\), we construct an MPE query by randomly selecting a subset \( \mathbf{Q} \subset \mathbf{X} \) such that \(|\mathbf{Q}| = qr \cdot |\mathbf{X}|\). The remaining variables form the evidence set \(\mathbf{E} = \mathbf{X} \setminus \mathbf{Q}\). We then sample a random complete assignment \(\mathbf{x}\) over \(\mathbf{X}\) and project it onto \(\mathbf{E}\) to generate the observed evidence.

We then compute a near-optimal assignment $\hat{\mathbf{x}}$ by running an anytime MPE solver on the constructed query within a fixed time budget $B$.  
Next, we collect a set of training states $\mathbf{S}$ by running a local search algorithm to explore the assignment space; this yields a diverse mix of high- and low-probability states under $P_{\mathcal{M}}(\mathbf{x})$, helping the model generalize beyond greedy regions.

For each state $\mathbf{x} \in \mathbf{S}$ and its 1-flip neighborhood $\mathcal{N}(\mathbf{x})$, we label every neighbor $\mathbf{x}' \in \mathcal{N}(\mathbf{x})$ as $1$ if $d_H(\mathbf{x}',\hat{\mathbf{x}}) < d_H(\mathbf{x},\hat{\mathbf{x}})$ and $0$ otherwise.
These labels indicate whether a single flip moves closer to the approximate optimum. Training the neural network on this data yields a probabilistic scorer $\hat{p}_{\downarrow}(\mathbf{x},\mathbf{x}')$ estimating the chance that a candidate move decreases Hamming distance.

At inference, the search simply selects the neighbor maximizing $\hat{p}_{\downarrow}$. This \emph{amortized lookahead strategy} guides local search toward long-term utility and accelerates convergence to near-optimal solutions.

\subsection{Neural Scoring Model}
\label{subsec:NN}

Given the dataset from Section~\ref{subsec:data}, we train a neural network surrogate to estimate the \textit{probability} that a neighbor $\mathbf{x}' \in \mathcal{N}(\mathbf{x})$ moves a state $\mathbf{x}$ closer to the reference assignment $\hat{\mathbf{x}}$. Figure~\ref{fig:architecture} provides an overview of the architecture: embeddings of the current assignment and candidate flips are contextualized through attention, then processed by a feed-forward encoder to produce neighbor-level scores. 

Each neighbor differs from the current assignment by exactly one variable flip.  
We therefore encode inputs as:  
(i) the current state $\mathbf{x}$, represented as a set of variable–value pairs, and  
(ii) for each neighbor $\mathbf{x}'$, the single query variable–value pair defining its difference from $\mathbf{x}$.

Taking inspiration from recent work~\citep{malhotra2025learning}, we propose an attention-based neural network architecture in which the variable–value pairs are embedded into learnable $d$-dimensional vectors~\citep{Mikolov2013EfficientEO}. To model dependencies between the current state and a candidate move, we employ an attention-based architecture~\citep{Attention}: embeddings of the current state act as \emph{keys} and \emph{values}, while the candidate move embeddings serve as the \emph{queries}. This design contextualizes each potential flip with respect to the full assignment and evidence, allowing the network to capture dependencies between observed evidence and promising variable flips.

Each contextualized neighbor embedding is concatenated with its raw embedding and passed through an encoder of fully connected layers with ReLU activations, residual connections~\citep{Residual}, and dropout regularization~\citep{srivastava_waypreventneural_2014}. The encoder produces logits, one per neighbor, which a multi-label classification head converts into probabilities indicating whether the flip reduces the Hamming distance.

\noindent\textbf{Multi-Label Classification Loss.}
Recall that for every search state $\mathbf{x}\in\mathbf{S}$ and neighbor $\mathbf{x}'\in\mathcal{N}(\mathbf{x})$, the label $y(\mathbf{x},\mathbf{x}')$ equals $1$ if the move decreases the Hamming distance to the reference assignment and $0$ otherwise.

Each neighbor $\mathbf{x}'$ is created by flipping a single query variable 
$Q_i\in\mathbf{Q}$ to a new value $q\in D_{Q_i}\setminus\{x_{Q_i}\}$.  
The classifier outputs $\hat{p}_{\downarrow}(\mathbf{x},\mathbf{x}')\in[0,1]$, 
interpreted as the probability that this flip decreases the Hamming distance.

We train the model with a binary cross-entropy loss applied independently to all 
neighbors:
\begin{align}
\mathcal{L}
&=
-\!\!\sum_{\mathbf{x}\in\mathbf{S}}
  \sum_{\mathbf{x}'\in\mathcal{N}(\mathbf{x})}
  y(\mathbf{x},\mathbf{x}') \log \hat{p}_{\downarrow}(\mathbf{x},\mathbf{x}') \notag\\
&\quad
-\!\!\sum_{\mathbf{x}\in\mathbf{S}}
  \sum_{\mathbf{x}'\in\mathcal{N}(\mathbf{x})}
  \bigl(1-y(\mathbf{x},\mathbf{x}')\bigr)
  \log \bigl(1-\hat{p}_{\downarrow}(\mathbf{x},\mathbf{x}')\bigr)
\end{align}



This formulation provides dense supervision for every 1-flip neighbor, 
and trains the network to score moves by their likelihood of reducing 
the Hamming distance to the reference solution.

\subsection{Inference-Time Search Guidance}
\label{sec:infer}

Our training framework is based on Hamming distance. An oracle neighbor selector always chooses a move that reduces the distance to the optimal assignment $\mathbf{x}^\ast$, providing a clean supervisory signal in which each neighbor can be labeled as \emph{good} (reduces $d_H$) or \emph{bad} (does not). Using solved or near-optimal MPE instances, the neural network is trained to approximate this oracle. At inference time, however, the true solution $\mathbf{x}^\ast$ is unknown, and predictions derived from approximate supervision are inherently noisy. When used in isolation, the learned guidance can therefore misdirect the search, particularly far from high-quality regions of the state space.

In contrast, the log-likelihood gain $
\mathcal{S}_{\mathrm{LL}}(\mathbf{x}' \mid \mathbf{x})$
is exact and always available at test time. It measures the immediate objective improvement of moving from $\mathbf{x}$ to a neighbor $\mathbf{x}'$. While reliable, this signal is purely greedy and frequently leads to entrapment in poor local optima. For notational convenience, we write $\mathcal{S}_{\mathrm{LL}}(\mathbf{x}')$ when the current state is clear from context.

Neither signal is sufficient on its own: neural guidance offers long-term lookahead but is imperfect, while log-likelihood gain is accurate but myopic. We therefore combine the two into a unified decision rule that balances short-term improvement with long-term promise. The neural score $\mathcal{S}_{\mathrm{NN}}(\mathbf{x}') \in [0,1]$ predicts the probability that a move decreases the Hamming distance to a high-quality solution, while the log-likelihood gain anchors the search in objective-improving directions. Since $\mathcal{S}_{\mathrm{LL}}(\mathbf{x}')$ is unbounded and $\mathcal{S}_{\mathrm{NN}}(\mathbf{x}')$ lies in $[0,1]$, we normalize $\mathcal{S}_{\mathrm{LL}}$ across the neighborhood using min--max scaling:
\[
\tilde{\mathcal{S}}_{\mathrm{LL}}(\mathbf{x}') \;=\;
\frac{\mathcal{S}_{\mathrm{LL}}(\mathbf{x}') - \min_{\mathbf{z}\in\mathcal{N}(\mathbf{x})}\mathcal{S}_{\mathrm{LL}}(\mathbf{z})}
     {\max_{\mathbf{z}\in\mathcal{N}(\mathbf{x})}\mathcal{S}_{\mathrm{LL}}(\mathbf{z}) - \min_{\mathbf{z}\in\mathcal{N}(\mathbf{x})}\mathcal{S}_{\mathrm{LL}}(\mathbf{z})}.
\]

The final selection score is given by:
\[
\mathcal{S}_{\mathrm{final}}(\mathbf{x}') \;=\;
(1-\lambda)\,\tilde{\mathcal{S}}_{\mathrm{LL}}(\mathbf{x}') +
\lambda\,\mathcal{S}_{\mathrm{NN}}(\mathbf{x}'),
\]
where $\lambda \in [0,1]$ balances short-term objective improvement with learned predictive guidance.
Setting $\lambda=0$ recovers standard greedy search, while $\lambda=1$ relies entirely on the neural surrogate.
In practice, $\lambda$ is tuned on a validation set.
At each search step, the neighbor $\mathbf{z}$ with the highest combined score $\mathcal{S}_{\mathrm{final}}(\mathbf{z})$ is chosen.

This \emph{neural lookahead} strategy augments classical best-improvement with a learned signal intended to capture longer-horizon utility.
The convex weighting provides a simple and interpretable mechanism to interpolate between the reliability of log-likelihood and the foresight of the neural predictor.
While a convex combination does not provide a universal dominance guarantee over its endpoints, selecting $\lambda$ on a validation set can improve robustness across instances and yields consistent gains in our experiments.

\section{Experiments}\label{sec:expts}

We evaluate the proposed method, \textsc{BEACON}\textemdash \underline{Be}yond-greedy \underline{E}stimation with \underline{A}ttention for \underline{C}onvergence to \underline{O}ptimal \underline{N}eighborhoods\textemdash which trains a neural network (Figure~\ref{fig:architecture}) on data generated by the procedure in Section~\ref{subsec:data} and performs inference using the method described in Section~\ref{sec:infer}. We assess \textsc{BEACON} by integrating it into two widely used local search strategies, replacing their standard neighbor scoring and selection rules with our learned scoring function.

\paragraph{Local Search Strategies.}
The first baseline is a conventional best-improvement local search that intensifies the search by repeatedly selecting the highest-scoring neighbor and restarting when trapped in a local optimum (\textsc{Greedy}). The second is the \textsc{GLS+} algorithm~\citep{gls+}, which augments best-improvement search with guided local search penalties. These methods are established references for evaluating neighbor-selection strategies.

\paragraph{BEACON Integration.}
We apply the inference procedure described in Section~\ref{sec:infer} to guide neighbor selection using neural network predictions, enabling the search to prioritize promising regions of the solution space. We denote the resulting variants as \textsc{BEACON-Greedy} and \textsc{BEACON-GLS+}. All experiments start from uniformly random initial assignments.\footnote{Although mini-bucket initialization is known to benefit local search for MPE~\citep{gls+}, we use random initialization to study the effect of neighbor selection in isolation.}

\subsection{Experimental Setup}


\paragraph{Amortized Inference Regime.}
Unlike standard competition protocols that evaluate a single MPE query per network, we evaluate the amortized performance of neural solvers on a fixed graphical model. The proposed neural network supports amortized inference by accumulating and reusing structural knowledge across multiple MPE queries on the same PGM, which is a central motivation for our evaluation setting. Accordingly, we construct a dataset of 1{,}000 distinct MPE instances per graphical model, split into 800 training, 100 validation, and 100 test instances. This protocol evaluates the ability of the learned model to generalize to unseen evidence patterns on a fixed graph and aligns with recent learning-based approaches to inference in PGMs~\citep{pmlr-v180-agarwal22a, arya_2024_solveconstraineda, arya2024_nn_mpe_nips, pmlr-v258-arya25a, malhotra2025learning}.

\paragraph{Benchmarks and Implementation Details.}
We evaluate on a diverse set of 25 high-treewidth PGMs from the UAI inference competitions~\citep{elidan_2010_2010, UAICompetition2022}, covering domains such as medical diagnosis (Promedas), genetic linkage analysis (Pedigree), Ising Models (Grids) and weighted constraint satisfaction problems (WCSP). The networks contain between 351 and 6{,}400 variables and up to 100{,}710 factors. For each network, we generate MPE instances via Gibbs sampling~\citep{koller-pgm} by designating 5\% to 20\% of variables as evidence, corresponding to a query ratio $qr \in [0.8, 0.95]$. We collect supervised training trajectories using the procedure in algorithm~\ref{algo:data}, with an anytime solver budget of $B=300$s and a local search step limit of $stl=2500$. Based on preliminary experiments, we use GLS+ with restarts at regular intervals for anytime MPE solving, as it achieves the best average performance across the benchmark suite.

The neural network uses 256-dimensional node embeddings, two multi-head attention layers, and ten skip-connection blocks. Each dense layer contains 512 units with ReLU activations and a dropout rate of 0.1~\citep{srivastava_waypreventneural_2014}. We train the model using Adam~\citep{kingma_stochasticoptimization_2017} with a learning rate of $2\times10^{-4}$, an exponential decay rate of $0.99$, and a batch size of 256. We apply early stopping after five epochs without validation improvement, with a maximum of 50 epochs. We tune the hyperparameter $\lambda$ over the candidate set $\{0.2, 0.5, 0.7, 1.0\}$ and select the value that performs best on the validation set. All models are implemented in PyTorch~\citep{paszke2019pytorch} and executed on a single NVIDIA A40 GPU. Additional design choices, implementation details, and hyperparameter settings are provided in Appendix~\ref{app:data} and \ref{app:hyperparameters}.

\subsection{Empirical Evaluation}
We evaluate all methods on a held-out test set of 100 queries for each graphical model. Each algorithm runs for a maximum of 4{,}000 search steps per query. All local search procedures aim to maximize the log-likelihood of the assignment, \( \ln p_\mathcal{M}(\mathbf{e}, \mathbf{q}) \), as the search progresses. We compare methods by recording the log-likelihood achieved at fixed step intervals and report two metrics: (i) \emph{win percentage} and (ii) \emph{average percentage improvement in log-likelihood} relative to the baseline.

\paragraph{Win Percentage.}
For each query and step budget, we compare the log-likelihood attained by the \textsc{BEACON}-augmented method against its corresponding baseline. We score a win, tie, and loss as 1, 0.5, and 0, respectively, and report the win percentage as the average score across the 100 test queries. Figure~\ref{fig:baseline-wins} visualizes these results using heatmaps, where green indicates improvement over the baseline, red indicates degradation, and darker shades represent larger differences.


\begin{figure}[t!]
    \centering
    \includegraphics[width=\linewidth]{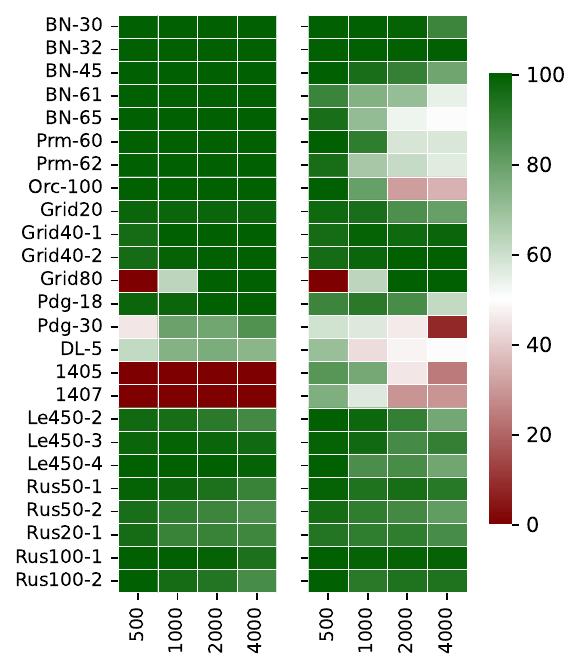}
    \caption{Win percentage heatmaps comparing our methods to their respective baselines: (left) \textsc{BEACON-Greedy} vs.\ \textsc{Greedy} and (right) \textsc{BEACON-GLS+} vs.\ \textsc{GLS+}. Columns show datasets; rows show search steps. Green indicates that the \textsc{BEACON}-enhanced method achieves higher log-likelihood than its baseline in majority of the cases ($>$ 50 points), red indicates lower performance, and darker shades represent larger differences.}
    \label{fig:baseline-wins}
\end{figure}


Across most graphical models, augmenting local search with \textsc{BEACON} consistently improves performance, as reflected by the predominantly green regions in both heatmaps. The gains are strongest early in the search, where neural-guided neighbor selection helps escape poor initial assignments and accelerates convergence. Although baseline methods narrow the gap with more steps, \textsc{BEACON}-enhanced variants typically retain an advantage even at 4{,}000 steps, indicating sustained effectiveness of the learned guidance.

For \textsc{BEACON-Greedy} (left panel of Figure~\ref{fig:baseline-wins}), the improvements are often substantial, with dark green regions dominating across most models with only minor regressions on some WCSP instances (dark red). \textsc{BEACON-GLS+} (right panel) delivers more uniform gains by overcoming the lack of amortization in standard GLS+, which must relearn penalties for each query. By transferring learned structural knowledge across queries, \textsc{BEACON} reaches high-quality solutions earlier much faster which is most pronounced on large WCSP and RUS benchmarks where dark green regions persist. Overall, these results demonstrate that \textsc{BEACON} both enhances greedy local search by escaping local optima and highlights the benefits of amortized, transferable guidance across queries.

\begin{table}[t]
\caption{Percentage improvement of \textsc{BEACON-Greedy} over \textsc{Greedy} at different step budgets.}
\centering
\setlength{\tabcolsep}{3pt}
\renewcommand{\arraystretch}{0.95}
\resizebox{0.8\linewidth}{!}{
\scriptsize
\begin{tabular}{l*{4}{S[table-format=+2.2]}}
\toprule
\multicolumn{1}{c}{Steps} & \multicolumn{1}{c}{500} & \multicolumn{1}{c}{1000} & \multicolumn{1}{c}{2000} & \multicolumn{1}{c}{4000} \\
\midrule
BN-30 & 52.47 & 53.62 & 54.98 & 55.41 \\
BN-32 & 69.64 & 65.51 & 66.90 & 67.61 \\
BN-45 & 62.50 & 67.15 & 69.10 & 70.43 \\
BN-61 & 57.98 & 64.96 & 68.35 & 70.80 \\
BN-65 & 75.45 & 76.58 & 77.06 & 77.41 \\
Prm-60 & 77.93 & 79.91 & 83.69 & 84.71 \\
Prm-62 & 81.45 & 84.13 & 85.72 & 86.09 \\
Orc-100 & 41.71 & 41.30 & 41.94 & 41.91 \\
Grid20 & 11.45 & 10.07 & 9.17 & 8.48 \\
Grid40-1 & 3.57 & 5.81 & 6.10 & 6.09 \\
Grid40-2 & 4.05 & 7.36 & 7.86 & 7.68 \\
Grid80 & -1.02 & 0.01 & 5.60 & 6.06 \\
Pdg-18 & 16.50 & 16.86 & 17.70 & 16.40 \\
Pdg-30 & -0.88 & 4.03 & 3.70 & 3.80 \\
DL-5 & 1.30 & 9.26 & 9.47 & 7.74 \\
1405 & -1.81 & -1.92 & -1.97 & -1.92 \\
1407 & -0.39 & -0.40 & -0.39 & -0.38 \\
Le450-2 & 1.17 & 0.60 & 0.39 & 0.25 \\
Le450-3 & 7.25 & 5.02 & 3.89 & 3.15 \\
Le450-4 & 25.19 & 20.58 & 16.90 & 13.98 \\
Rus20-1 & 7.24 & 3.30 & 2.06 & 1.43 \\
Rus50-1 & 7.60 & 4.16 & 2.83 & 1.67 \\
Rus50-2 & 6.46 & 3.54 & 2.30 & 1.55 \\
Rus100-1 & 6.83 & 3.30 & 1.93 & 0.93 \\
Rus100-2 & 6.35 & 2.88 & 1.56 & 0.84 \\
\bottomrule
\end{tabular}
}
\label{tab:pct-improvement-greedy}
\end{table}

\begin{table}[t]
\caption{Percentage improvement of \textsc{BEACON-GLS+} over GLS+ at different step budgets.}
\centering
\setlength{\tabcolsep}{3pt}
\renewcommand{\arraystretch}{0.95}
\resizebox{0.8\linewidth}{!}{
\scriptsize
\begin{tabular}{l*{4}{S[table-format=+2.2]}}
\toprule
\multicolumn{1}{c}{Steps} & \multicolumn{1}{c}{500} & \multicolumn{1}{c}{1000} & \multicolumn{1}{c}{2000} & \multicolumn{1}{c}{4000} \\
\midrule
BN-30 & 49.42 & 36.56 & 19.29 & 8.21 \\
BN-32 & 46.23 & 30.28 & 13.36 & 5.52 \\
BN-45 & 42.62 & 29.19 & 17.55 & 8.55 \\
BN-61 & 26.74 & 15.11 & 7.26 & 1.46 \\
BN-65 & 20.19 & 4.78 & -0.36 & -0.52 \\
Prm-60 & 82.00 & 33.07 & -2.06 & -1.22 \\
Prm-62 & 45.55 & -0.42 & 0.35 & 1.58 \\
Orc-100 & 35.99 & 8.48 & -1.55 & -0.90 \\
Grid20 & 5.20 & 2.21 & 0.93 & 0.40 \\
Grid40-1 & 3.28 & 4.38 & 3.07 & 2.14 \\
Grid40-2 & 3.69 & 5.78 & 4.00 & 2.42 \\
Grid80 & -1.02 & 0.01 & 5.24 & 3.75 \\
Pdg-18 & 9.70 & 4.79 & 2.53 & 0.52 \\
Pdg-30 & 0.92 & 0.48 & -0.38 & -2.04 \\
DL-5 & 4.97 & -5.14 & -3.19 & -2.52 \\
1405 & 0.94 & 0.38 & 0.00 & -0.12 \\
1407 & 0.09 & 0.03 & -0.03 & -0.01 \\
Le450-2 & 0.74 & 0.65 & 0.32 & 0.11 \\
Le450-3 & 4.00 & 2.82 & 0.85 & 0.74 \\
Le450-4 & 15.09 & 5.08 & 2.02 & 1.04 \\
Rus20-1 & 6.67 & 4.23 & 2.39 & 1.67 \\
Rus50-1 & 7.48 & 4.36 & 2.80 & 1.90 \\
Rus50-2 & 6.43 & 3.89 & 2.51 & 1.78 \\
Rus100-1 & 6.82 & 2.00 & 1.41 & 1.17 \\
Rus100-2 & 6.35 & 1.47 & 1.00 & 0.95 \\
\bottomrule
\end{tabular}
}
\label{tab:pct-improvement-gls}
\end{table}

\paragraph{Percentage Increase in the Log-Likelihood.}
Tables~\ref{tab:pct-improvement-greedy} and~\ref{tab:pct-improvement-gls} report the percentage improvement in log-likelihood achieved by \textsc{BEACON-Greedy} and \textsc{BEACON-GLS+} relative to their respective baselines. We compute this metric as \(\frac{1}{N}\sum_{i=1}^{N}\frac{\mathcal{LL}_{D}^{(i)}-\mathcal{LL}_{S}^{(i)}}{|\mathcal{LL}_{S}^{(i)}|}\times100\), where \(\mathcal{LL}_{S}^{(i)}\) and \(\mathcal{LL}_{D}^{(i)}\) denote the baseline and \textsc{BEACON}-enhanced log-likelihood for instance \(i\), respectively.





\eat{Tables~\ref{tab:pct-improvement-greedy} and~\ref{tab:pct-improvement-gls} report the percentage improvement in log-likelihood of \textsc{BEACON-Greedy} over \textsc{Greedy} and \textsc{BEACON-GLS+} over \textsc{GLS+}, respectively.}

Each row in both tables corresponds to a graphical model, identified in the first column, followed by four entries reporting results for step budgets of 500, 1000, 2000, and 4000. Positive values indicate that the \textsc{BEACON}-augmented method attains higher log-likelihood than its baseline, while negative values indicate degradation. Larger magnitudes correspond to greater relative changes. Because log-likelihood scales vary across models due to differences in partition functions, values are not comparable across rows.

Table~\ref{tab:pct-improvement-greedy} shows that \textsc{BEACON-Greedy} achieves substantial improvements over the \textsc{Greedy} baseline, typically remaining positive as search progresses. This confirms that neural guidance effectively mitigates premature convergence to poor local optima, yielding faster progress and higher-quality solutions. Table~\ref{tab:pct-improvement-gls} demonstrates moderate but consistent gains for \textsc{BEACON-GLS+}, particularly in early search stages. By leveraging amortization to transfer offline-learned structural knowledge, \textsc{BEACON} accelerates convergence and improves final solution quality. This advantage remains robust even as the step budget increases, particularly on challenging WCSP and RUS benchmarks where \textsc{BEACON-GLS+} maintains a clear performance lead.

\paragraph{Summary.}
Experiments on 25 high-treewidth probabilistic graphical models show that augmenting local search with \textsc{BEACON} consistently improves performance. When integrated into \textsc{Greedy} and \textsc{GLS+}, \textsc{BEACON}-enhanced methods reach higher-quality solutions in significantly fewer steps, whereas baseline methods require substantially more iterations to attain comparable log-likelihoods. The advantage is most pronounced early in the search, where neural guidance enables rapid escape from poor initial regions and faster convergence to high-quality assignments. These results demonstrate that attention-based neighbor selection accelerates convergence and often yields superior solutions compared with conventional best-improvement local search.


We also provide a comparison of \textsc{BEACON} against state-of-the-art non–local-search solvers, including \textsc{DAOOPT}~\citep{marinescu2010daoopt} and \textsc{TOULBAR2}~\citep{givry:hal-04021879}, in Appendix~\ref{app:exact-solvers-comparison}.

\section{Conclusion}

We propose a neural-guided lookahead strategy for neighbor selection in local search for MPE inference in probabilistic graphical models. By augmenting best-improvement with an optimality-aware neural scoring function, the method guides search toward higher log-likelihood regions and reduces premature convergence to poor local optima. The learned guidance captures transferable structural information from (near-)optimal assignments, enabling more efficient test-time search. As a result, high-quality solutions are reached significantly faster than with methods such as GLS+, which must relearn heuristics for each query. Experiments on high-treewidth PGMs demonstrate consistent performance gains across diverse models, highlighting learned guidance as an effective drop-in enhancement to classical local search.

\noindent\textbf{Limitations and Future Work.}
Our approach entails several natural trade-offs. As with most learned components, performance can be sensitive to distributional shifts and the choice of the hyperparameter $\lambda$, although we find the method to be stable across a wide range of settings in practice. Training relies on supervision from approximate solvers, which may introduce some bias but enables scalability to problem instances where exact MPE solutions are unavailable. The current one-step Hamming-based guidance focuses on local improvements and may miss longer-horizon gains. Additionally, scoring all neighbors with a neural model can add nontrivial computational overhead in models with large variable domains. Finally, unlike exact inference methods, convergence guarantees depend on predictor accuracy. These considerations point to several promising directions for future work, including richer and less biased supervision sources, multi-step lookahead architectures, adaptive weighting mechanisms, and extensions to broader classes of discrete optimization problems.

\section*{Acknowledgements}
This work was supported in part by the DARPA CODORD program under contract number HR00112590089, the DARPA Assured Neuro Symbolic Learning and Reasoning (ANSR) Program under contract number HR001122S0039, the National Science Foundation grant IIS-1652835, and the AFOSR award FA9550-23-1-0239.

\bibliographystyle{plain}
\bibliography{references}

\newpage
\appendix
\onecolumn
\section{Proofs of the Theoretical Results}
\label{app:proofs}

\subsection{Proof of Theorem~\ref{thm:drift}: (Convergence)}
\begin{proof}
Let $\{x_t\}_{t \ge 0}$ denote the sequence of assignments produced by
one–flip local search, and let $h_t = d_H(x_t, x^\ast)$ be the Hamming
distance to the optimal assignment $x^\ast$.  
At each non–optimal state, by assumption, the executed move
reduces the distance by one with probability $\alpha$ and increases it
by one with probability $(1 - \alpha)$.  Thus the conditional
expectation satisfies
\[
\mathbb{E}[h_{t+1} \mid h_t]
  = h_t - (2\alpha - 1).
\tag{3}
\]

Taking total expectation on both sides and unrolling this recurrence
yields
\[
\mathbb{E}[h_t]
  = \mathbb{E}[h_{t-1}] - (2\alpha - 1)
  = h_0 - t(2\alpha - 1),
\tag{4}
\]
valid as long as the process has not yet been absorbed at $h_t = 0$.
Hence the expected distance decreases linearly with $t$
at rate $(2\alpha - 1)$ per step.

Now define the stopping time
\[
\tau^\ast := \inf\{t \ge 0 : h_t = 0\},
\]
which represents the first time the process reaches the optimum.
Because $h_t$ is non–negative and decreases in expectation by a fixed
amount at each step (Equation~3), the process has a strictly positive
drift toward zero.  Intuitively, each step removes on average
$(2\alpha - 1)$ units of distance, and the total distance to remove is
$h_0$.  Taking expectations and applying the law of total expectation
gives
\[
\mathbb{E}[h_{\tau^\ast}]
  = \mathbb{E}\!\left[h_0 - (2\alpha - 1)\tau^\ast\right]
  = 0,
\]
since $h_{\tau^\ast}=0$ at absorption.  Rearranging yields
\[
(2\alpha - 1)\,\mathbb{E}[\tau^\ast]
   = h_0.
\]
Because the decrease per step is an expectation (not deterministic),
this identity gives an upper bound on the true expected hitting time:
\[
\mathbb{E}[\tau^\ast] \le \frac{h_0}{2\alpha - 1}.
\tag{5}
\]

\paragraph{High--probability behavior.}
Since $h_t$ changes by at most one per step and exhibits a positive drift of $(2\alpha-1)$ toward zero when $\alpha>\tfrac12$,
standard concentration results for drift processes with bounded increments (e.g., Azuma--Hoeffding or additive drift tail bounds)
imply that the hitting time $\tau^\dagger$ concentrates around its expectation and has an exponentially decaying tail above $\mathbb{E}[\tau^\dagger]$.
For clarity, we focus on the expected hitting-time guarantee in Theorem~\ref{thm:drift} and omit an explicit tail expression.


\end{proof}

\subsection{Data Generation for Neighbor Scoring}
\label{app:data}

Our goal in this section is to describe a concrete pipeline for constructing supervised training data for
\emph{neighbor scoring}. The final learning problem is the following: given a current assignment
$\mathbf{x}$ (consistent with evidence) and a 1-flip neighbor $\mathbf{x}'\in\mathcal{N}(\mathbf{x})$,
we want a model to score $\mathbf{x}'$ so that, within the neighborhood $\mathcal{N}(\mathbf{x})$,
moves that are \emph{distance-reducing} (with respect to a high-quality target assignment) are ranked
above moves that are not.

\paragraph{MPE queries and notation.}
Let $\mathcal{M}$ be a probabilistic graphical model over variables $\mathbf{X}$ with factors
$\mathbf{F}$ and evidence $\mathbf{e}$.
For a query set $\mathbf{Q}\subset \mathbf{X}$, let $\mathbf{E}=\mathbf{X}\setminus \mathbf{Q}$
denote the evidence variables. Evidence variables are fixed to $\mathbf{e}$, and local search only
modifies the query variables.
We write $\mathbf{x}$ for a complete assignment to $\mathbf{X}$ that is consistent with evidence
(i.e., $\mathbf{x}_{\mathbf{E}}=\mathbf{e}$), and we write $\mathbf{x}_{\mathbf{Q}}$ for its restriction
to the query variables.

To measure progress toward a target assignment on the query variables, we use Hamming distance
restricted to $\mathbf{Q}$:
\[
d_H(\mathbf{x},\mathbf{y})
\;=\;
\left|\left\{Q_i\in\mathbf{Q}\;:\; x_{Q_i}\neq y_{Q_i}\right\}\right|.
\]
Because $\mathbf{E}$ is fixed, any neighbor move changes exactly one query variable and therefore
changes $d_H(\cdot,\cdot)$ by exactly $1$.

\paragraph{Reference assignments for supervision.}
In principle, one may define supervision with respect to the optimal MPE solution $\mathbf{x}^\ast$.
In practice, however, computing $\mathbf{x}^\ast$ is infeasible for large, high-treewidth models.
Instead, for each training query we compute a \emph{high-quality reference assignment}
$\hat{\mathbf{x}}$ using an anytime MPE solver with a fixed computational budget.
This reference solution provides a stable surrogate target that induces a well-defined notion of
distance-reducing moves.

Formally, for a state $\mathbf{x}$ and a neighbor $\mathbf{x}' \in \mathcal{N}(\mathbf{x})$, we define the
binary supervision signal
\[
y(\mathbf{x},\mathbf{x}')
\;=\;
\mathbb{I}\!\left[d_H(\mathbf{x}', \hat{\mathbf{x}}) < d_H(\mathbf{x}, \hat{\mathbf{x}})\right],
\]
where $\mathbb{I}[\cdot]$ denotes the indicator function.
Since $\mathbf{x}'$ differs from $\mathbf{x}$ by a single query-variable flip, this label simply indicates
whether the move advances one step toward the reference solution in Hamming distance.

\paragraph{Trajectory-based state collection (distribution matching).}
A central design choice concerns the selection of training states $\mathbf{S}$ on which neighborhoods are
labeled.
While one could sample assignments uniformly or from the model distribution $P_{\mathcal{M}}$, such states
do not reflect those encountered during inference.
Local search concentrates rapidly in structured regions of the state space (e.g., basins of attraction,
plateaus, and near-optimal ``funnels'') whose neighborhood structure differs markedly from that of random
assignments.

To reduce train--test mismatch, we therefore construct $\mathbf{S}$ by executing a local search procedure
under the same evidence $\mathbf{e}$ and collecting the intermediate states along its trajectory.
This trajectory-based sampling yields training data that more closely matches the inference-time
distribution.
Empirically, this choice is crucial for learning a scorer whose \emph{within-neighborhood ranking} remains
reliable in the regions of the search space that dominate practical inference.

\paragraph{Data collection procedure.}
Algorithm~\ref{algo:data} summarizes the complete pipeline. For each sample $i=1,\dots,N$:
\begin{enumerate}
    \item \textbf{Generate an MPE query.}
    We sample a complete assignment $\tilde{\mathbf{x}}\sim P_{\mathcal{M}}$ and select a random query set
    $\mathbf{Q}\subset\mathbf{X}$ of size $|\mathbf{Q}|=qr\cdot|\mathbf{X}|$.
    The remaining variables $\mathbf{E}=\mathbf{X}\setminus\mathbf{Q}$ form the evidence set, and we set
    evidence to $\mathbf{e}=\tilde{\mathbf{x}}_{\mathbf{E}}$.
    This procedure yields consistent evidence assignments and avoids pathological evidence choices.

    \item \textbf{Compute a high-quality reference solution.}
    We run an anytime MPE solver on $(\mathcal{M},\mathbf{e})$ for a fixed time budget $B$ and obtain a
    reference assignment $\hat{\mathbf{x}}$ (consistent with $\mathbf{e}$).
    The intent is not to guarantee optimality, but to obtain a strong and coherent target that provides
    stable labels.

    \item \textbf{Collect trajectory states.}
    We then run a 1-flip local search procedure under the same evidence $\mathbf{e}$ for a fixed step limit
    $stl$ and collect the visited states into a set (or multiset) $\mathbf{S}$.
    The local search used here can be the same template used at inference time (e.g., best-improvement or a
    mild stochastic variant), since the primary goal is to sample states from the distribution encountered
    during search.

    \item \textbf{Label neighbors at each collected state.}
    For each $\mathbf{x}\in\mathbf{S}$, we enumerate its 1-flip neighborhood $\mathcal{N}(\mathbf{x})$
    (flipping one query variable to an alternative value) and assign labels
    $y(\mathbf{x},\mathbf{x}')\in\{0,1\}$ using the Hamming-distance criterion above.
    The resulting record stores $(\mathbf{e},\mathbf{x},\{(\mathbf{x}',y(\mathbf{x},\mathbf{x}')) :
    \mathbf{x}'\in\mathcal{N}(\mathbf{x})\})$.
\end{enumerate}

\paragraph{Learning objective and inference-time usage.}
Algorithm~\ref{algo:data} produces a dataset of local neighborhoods annotated with binary labels indicating
whether a candidate move reduces Hamming distance to the reference assignment $\hat{\mathbf{x}}$.
Training on this data yields a learned scoring function that, given a state $\mathbf{x}$ and a neighbor
$\mathbf{x}' \in \mathcal{N}(\mathbf{x})$, outputs a value in $[0,1]$ interpreted as the likelihood that the
move is distance-reducing:
\[
\hat{p}_{\downarrow}(\mathbf{x},\mathbf{x}')
\;\approx\;
\Pr\!\left[d_H(\mathbf{x}',\hat{\mathbf{x}})
=
d_H(\mathbf{x},\hat{\mathbf{x}})-1
\;\middle|\;
\mathbf{x}, \mathbf{x}'\right].
\]

At inference time, we use this model as a \emph{ranking function} within each neighborhood:
\[
\mathbf{x}_{t+1}
\;=\;
\arg\max_{\mathbf{z}\in\mathcal{N}(\mathbf{x}_t)}
\hat{p}_{\downarrow}(\mathbf{x}_t,\mathbf{z}).
\]
Importantly, the learned scores are not thresholded and need not be globally calibrated.
Only the relative ordering of neighbors matters.
This design aligns directly with our drift-style analysis, in which convergence depends on the frequency
with which selected moves are truly distance-reducing along the search trajectory, not on global calibration of scores across unrelated neighborhoods.

\paragraph{Design choices and practical considerations.}
Algorithm~\ref{algo:data} leaves several degrees of freedom; we summarize the most important ones here
and explain the rationale for our defaults.

\begin{itemize}

    \item \textbf{Evidence generation.}
    Evidence is generated by first sampling a complete assignment
    $\tilde{\mathbf{x}} \sim P_{\mathcal{M}}$ and setting $\mathbf{e}=\tilde{\mathbf{x}}_{\mathbf{E}}$.
    This ensures internal consistency and avoids degenerate evidence patterns with extremely low probability,
    yielding MPE queries that better reflect typical model behavior. We use Gibbs sampling procedure~\citep{koller-pgm} 
    to sample the complete assignments.

    \item \textbf{Choice of query ratio $qr$.}
    The query ratio controls the difficulty of the induced MPE subproblem.
    Small $qr$ yields easier instances (few decision variables), while large $qr$ yields harder instances.
    We treat $qr$ as a knob to generate a spectrum of difficult queries by randomizing it over a range $[0.8, 0.95]$ 
    to improve robustness.

    \item \textbf{Teacher quality via solver budget $B$.}
    The role of the anytime solver is to provide a high-quality reference assignment $\hat{\mathbf{x}}$ that
    induces coherent distance labels. If $B$ is too small, $\hat{\mathbf{x}}$ may be noisy, and the
    ``distance-reducing'' relation becomes a weak training signal.
    Increasing $B$ typically improves label quality but increases offline data-generation cost.
    In practice we fix $B=300$s and treat it as part of the teacher strength: stronger teachers yield cleaner
    supervision.

    \item \textbf{Collecting training states via local search.}
    The set $\mathbf{S}$ is meant to approximate the distribution of states encountered at inference time.
    If we instead sampled states randomly (uniformly, or even from $P_{\mathcal{M}}$), the induced
    neighborhoods would be qualitatively different from those visited by search, creating a train--test
    mismatch. Running local search under the same evidence $\mathbf{e}$ produces trajectory states that capture the
    basins, plateaus, and near-optimal regions that matter in practice.

    \item \textbf{Choice of search procedure for data collection.}
    The data-collection search does not need to be identical to the final inference-time search, but it
    should generate a reasonable coverage of states the model will later see.
    A purely greedy search can over-concentrate $\mathbf{S}$ near local optima, while a mildly stochastic
    variant (e.g., occasional random flips or random restarts) can improve coverage.
    We therefore use a simple local search procedure with a fixed step limit $stl=500$ that chooses neighbors greedily 50\% of the time and guided by $\mathbf{\hat{x}}$ remaining time with regular restarts.

    \item \textbf{Step limit $stl$ versus time limit.}
    We implement the data-collection budget as a step limit rather than wall-clock time to make the
    distribution of collected states less sensitive to hardware and implementation details.
    A fixed $stl$ also yields datasets of comparable size across instances.
    A time limit can be used instead, but then the number of collected states varies substantially.


    \item \textbf{State deduplication.}
    One can store $\mathbf{S}$ as a set (deduplicated) or as a multiset (keeping repeats).
    Deduplication increases diversity per sample, while repeats implicitly upweight frequently visited
    states.
    Because inference-time trajectories often revisit plateau states, keeping repeats can be beneficial.
    We treat this as an implementation choice; our default is to keep repeats within a trajectory and
    rely on batching/shuffling across samples to avoid overfitting.

    \item \textbf{Distance-based supervision.}
    The Hamming criterion yields a simple, local, and dense supervision signal: every neighbor is labeled,
    and in the binary classification case each move is unambiguously ``toward'' or ``away'' from the reference on exactly
    one bit.
    This is preferable to sparse supervision based only on objective improvement, which is often noisy and
    can be myopic.
    The learned scorer is used as a \emph{within-neighborhood} ranking function, so the key requirement is
    that distance-reducing moves tend to be ranked above non-reducing moves often enough along the
    trajectory.
    A convenient record format is
    $\{\mathbf{e}, \mathbf{x}, \mathcal{D}_x\}$ where
    $\mathcal{D}_x=\{(\mathbf{x}',y(\mathbf{x},\mathbf{x}')):\mathbf{x}'\in\mathcal{N}(\mathbf{x})\}$.
    This supports neighborhood-wise training and aligns with the inference-time operation of ranking
    neighbors within $\mathcal{N}(\mathbf{x})$.
\end{itemize}

\section{Description of the Probabilistic Graphical Models} \label{dataset}

Table~\ref{tab:datasets} summarizes the networks employed in our experiments, detailing the number of variables and factors for each model. The experiments were performed on high-treewidth probabilistic graphical models (PGMs) of varying scales from the UAI inference competitions \citep{elidan_2010_2010, UAICompetition2022}, with variable counts ranging from 351 to 6{,}400 and up to 100{,}710 factors. We additionally report the induced treewidth of each model, computed using DAOOPT~\citep{marinescu2010daoopt}, which ranges from 20 to 315. The table also lists the abbreviated model names used throughout the paper for ease of reference.

\begin{table}[ht!]
    \centering
    \caption{Summary of the networks used in our experiments, showing the number of variables, number of factors and Induced Treewidth for each model.}
    \label{tab:datasets}
    \begin{tabular}{|c|c|c|c|c|}
    \hline
    Network Name  & Short Name  & Number of Factors & Number of Variables & Induced Treewidth \\ 
    \hline
    \textbf{BN\_30} & BN-30 & 1156 & 1156 & 54\\
    \textbf{BN\_32} & BN-32 & 1444 & 1444 & 61\\
    \textbf{BN\_45} & BN-45 & 880 & 880 & 24\\
    \textbf{BN\_61} & BN-61 & 667 & 667 & 44\\
    \textbf{BN\_65} & BN-65 & 440 & 440 & 65\\
    \textbf{Promedas\_60} & Prm-60 & 1076 & 1076 & 32 \\
    \textbf{Promedas\_62} & Prm-62 & 639 & 639 & 31\\
    \textbf{or\_chain\_100} & Orc-100 & 1125 & 1125 & 61 \\
    \textbf{pedigree18} & Pdg-18 & 1184 & 1184 & 20 \\
    \textbf{pedigree30} & Pdg-30 & 1289 & 1289 & 21\\
    \textbf{grid20x20.f5.wrap} & Grid20 & 1200 & 400 & 46\\
    \textbf{grid40x40.f15.wrap} & Grid40-1 & 4800 & 1600 & 96\\
    \textbf{grid40x40.f10} & Grid40-2 & 4720 & 1600 & 54 \\
    \textbf{grid80x80.f10} & Grid80 & 19040 & 6400 & 108\\
    \textbf{driverlog05ac.wcsp} & DL-5 & 30038 & 351 & 68 \\
    \textbf{1405.wcsp} & 1405 & 18258 & 855 & 93\\
    \textbf{1407.wcsp} & 1407 & 21786 & 1057 & 93\\
    \textbf{le450\_5a\_2.wcsp} & Le450-2 & 5714 & 450 & 315 \\
    \textbf{le450\_5a\_3.wcsp} & Le450-3 & 5714 & 450 & 315 \\
    \textbf{le450\_5a\_4.wcsp} & Le450-4 & 5714 & 450 & 315 \\
    \textbf{rus\_20\_40\_9\_3} & Rus20-1 & 16950 & 854 & 30\\
    \textbf{rus\_50\_100\_4\_1} & Rus50-1 & 45360 & 944 & 60\\
    \textbf{rus\_50\_100\_6\_1} & Rus50-2 & 45360 & 944 & 60\\
    \textbf{rus\_100\_200\_1\_1} & Rus100-1 & 100710 & 1094 & 110\\
    \textbf{rus\_100\_200\_3\_3} & Rus100-2 & 100710 & 1094 & 110 \\
    \hline
    \end{tabular}
\end{table}

\section{Hyperparameter Details} \label{app:hyperparameters}


We used a consistent set of hyperparameters across all networks in our experiments. From the initial dataset of 1{,}000 samples, 900 were allocated for training and 100 for testing. The training set was further divided into 800 samples for training and 100 for validation. Note that we generate the set of training states $\mathbf{S}$ by running a local search algorithm that explores the assignment space using the training examples, resulting in a substantially larger effective training set for the neural network.

The neural networks in \textsc{BEACON} employ 256-dimensional embeddings~\citep{Mikolov2013EfficientEO}, two multi-head attention layers~\citep{Attention}, and ten skip-connection blocks~\citep{Residual}. Each dense layer comprises 512 units with a dropout~\citep{srivastava_waypreventneural_2014} rate of 0.1. Training was conducted using the Adam optimizer \citep{kingma_stochasticoptimization_2017} with a learning rate of $2 \times 10^{-4}$ and an exponential decay factor of $0.99$. Models were trained for up to 50 epochs with a batch size of 256, applying early stopping if the validation performance failed to improve for five consecutive epochs. After training, each model was evaluated on 100 MPE queries with a query ratio of $qr \in [0.8, 0.95]$, defined as the fraction of variables included in the query set.

Hyperparameters were selected using the validation dataset. The weighting parameter $\lambda$, which balances short-term gain and long-term performance, was tuned over $\{0.2, 0.5, 0.7, 1.0\}$ to maximize performance.

\section{Average Time Per Search Step}\label{sec:timeperstep}

\begin{table}[ht!]
    \centering
    \caption{Average time per step (in seconds) for local search algorithms. \textsc{BEACON} introduces additional cost due to neural scoring of neighbors but maintains low computational overhead across most networks.}
    \begin{tabular}{|c|c|c|c|c|}
    \hline
    Network Name & \textsc{Greedy} & \textsc{BEACON-Greedy} & \textsc{GLS+} & \textsc{BEACON-GLS+}\\
    \hline
    BN-30 & 0.0042 ± 0.0001 & 0.0071 ± 0.0001 & 0.0040 ± 0.0001 & 0.0073 ± 0.0000 \\
    BN-32 & 0.0030 ± 0.0001 & 0.0081 ± 0.0003 & 0.0031 ± 0.0001 & 0.0083 ± 0.0001 \\
    BN-45 & 0.0028 ± 0.0001 & 0.0062 ± 0.0003 & 0.0028 ± 0.0000 & 0.0063 ± 0.0001 \\
    BN-61 & 0.0039 ± 0.0001 & 0.0054 ± 0.0002 & 0.0027 ± 0.0001 & 0.0057 ± 0.0000 \\
    BN-65 & 0.0038 ± 0.0001 & 0.0049 ± 0.0002 & 0.0036 ± 0.0001 & 0.0055 ± 0.0000 \\
    Prm-60 & 0.0028 ± 0.0000 & 0.0069 ± 0.0001 & 0.0029 ± 0.0000 & 0.0071 ± 0.0001 \\
    Prm-62 & 0.0026 ± 0.0000 & 0.0053 ± 0.0001 & 0.0027 ± 0.0001 & 0.0056 ± 0.0003 \\
    Orc-100 & 0.0028 ± 0.0001 & 0.0073 ± 0.0004 & 0.0028 ± 0.0001 & 0.0073 ± 0.0002 \\
    Pdg-18 & 0.0063 ± 0.0001 & 0.0108 ± 0.0001 & 0.0067 ± 0.0001 & 0.0116 ± 0.0004 \\
    Pdg-30 & 0.0064 ± 0.0002 & 0.0113 ± 0.0001 & 0.0067 ± 0.0001 & 0.0119 ± 0.0001 \\
    Grid20 & 0.0024 ± 0.0001 & 0.0047 ± 0.0000 & 0.0027 ± 0.0000 & 0.0051 ± 0.0002 \\
    Grid40-1 & 0.0032 ± 0.0001 & 0.0088 ± 0.0003 & 0.0038 ± 0.0000 & 0.0092 ± 0.0001 \\
    Grid40-2 & 0.0032 ± 0.0000 & 0.0088 ± 0.0001 & 0.0037 ± 0.0000 & 0.0091 ± 0.0001 \\
    Grid80 & 0.0133 ± 0.0001 & 0.0364 ± 0.0008 & 0.0129 ± 0.0001 & 0.0365 ± 0.0010 \\
    DL-5 & 0.0135 ± 0.0007 & 0.0170 ± 0.0010 & 0.0161 ± 0.0001 & 0.0200 ± 0.0010 \\
    1405 & 0.0060 ± 0.0003 & 0.0106 ± 0.0004 & 0.0071 ± 0.0003 & 0.0120 ± 0.0006 \\
    1407 & 0.0066 ± 0.0001 & 0.0125 ± 0.0005 & 0.0079 ± 0.0004 & 0.0138 ± 0.0006 \\
    Le450-2 & 0.0025 ± 0.0000 & 0.0050 ± 0.0000 & 0.0029 ± 0.0000 & 0.0056 ± 0.0002 \\
    Le450-3 & 0.0037 ± 0.0001 & 0.0066 ± 0.0000 & 0.0040 ± 0.0000 & 0.0071 ± 0.0001 \\
    Le450-4 & 0.0046 ± 0.0000 & 0.0080 ± 0.0003 & 0.0053 ± 0.0000 & 0.0089 ± 0.0003 \\
    Rus20-1 & 0.0035 ± 0.0001 & 0.0068 ± 0.0000 & 0.0048 ± 0.0001 & 0.0078 ± 0.0000 \\
    Rus50-1 & 0.0064 ± 0.0001 & 0.0104 ± 0.0002 & 0.0090 ± 0.0001 & 0.0121 ± 0.0003 \\
    Rus50-2 & 0.0063 ± 0.0001 & 0.0103 ± 0.0001 & 0.0088 ± 0.0001 & 0.0121 ± 0.0003 \\
    Rus100-1 & 0.0155 ± 0.0001 & 0.0204 ± 0.0001 & 0.0188 ± 0.0001 & 0.0227 ± 0.0000 \\
    Rus100-2 & 0.0155 ± 0.0001 & 0.0204 ± 0.0000 & 0.0188 ± 0.0002 & 0.0227 ± 0.0002 \\
    \bottomrule
    \end{tabular}    
    \label{tab:time-per-step}
\end{table}

Table~\ref{tab:time-per-step} reports the average time per step for each local search algorithm. On average, all algorithms operate within a few milliseconds per step, enabling real-time inference. Although the baseline methods execute slightly faster than their counterparts augmented with \textsc{BEACON}, the additional computational cost remains minimal. The slowdown arises because, at each step, the neural network must compute scores for all neighbors \(\mathcal{N}(\mathbf{x})\) of the current assignment \(\mathbf{x}\), in addition to evaluating the log-likelihood gain. Even in the most computationally demanding cases, \textsc{BEACON} requires at most \(0.0365\,\mathrm{s}\) per step, indicating that its computational cost remains negligible for practical use.

\section{Comparison with Branch \& Bound based Solvers in Anytime Setting}\label{app:exact-solvers-comparison}

\begin{table}[ht!]
    \centering
    \caption{Average percentage gap (relative to \textsc{Greedy}) comparing local search methods and exact solvers under a 60 second limit. Exact solvers dominate on smaller networks, while \textsc{BEACON}-augmented methods yield better solutions on larger ones. Entries marked \texttt{-inf} indicate that no solution was found within the time limit. \textsc{Greedy} serves as the baseline (all values are zero by definition).}
    \begin{tabular}{|c|c|c|c|c|c|c|}
    \toprule
    \multicolumn{1}{|p{2cm}|}{\centering Network Name} & \textsc{Greedy} & \multicolumn{1}{|p{2cm}|}{\centering \textsc{BEACON-Greedy}} & \textsc{GLS+} & \multicolumn{1}{|p{2cm}|}{\centering \textsc{BEACON-GLS+}} & \textsc{DAOOPT} & \textsc{TOULBAR2}\\
    \midrule
    BN-30 & 0.00 ± 0.00 & 56.37 ± 4.26 & 80.69 ± 4.99 & 80.55 ± 4.92 & \textbf{81.25 ± 3.77} & \textbf{81.25 ± 3.77} \\
    BN-32 & 0.00 ± 0.00 & 67.20 ± 1.46 & 67.71 ± 0.76 & 69.47 ± 0.83 & \textbf{69.53 ± 0.70} & \textbf{69.53 ± 0.70} \\
    BN-45 & 0.00 ± 0.00 & 68.25 ± 4.61 & 82.42 ± 4.46 & 82.36 ± 4.33 & \textbf{83.43 ± 4.53} & \textbf{83.43 ± 4.53} \\
    BN-61 & 0.00 ± 0.00 & 65.62 ± 13.90 & 85.66 ± 4.84 & 85.14 ± 5.24 & \textbf{85.95 ± 4.85} & \textbf{85.95 ± 4.85}\\
    BN-65 & 0.00 ± 0.00 & 73.64 ± 4.30 & 80.17 ± 4.49 & 80.01 ± 4.52 & 79.28 ± 4.54 & \textbf{80.19 ± 4.50}\\
    Prm-60 & 0.00 ± 0.00 & 89.17 ± 5.81 & 92.42 ± 3.76 & 92.15 ± 4.02 & \textbf{92.53 ± 3.76} & \textbf{92.53 ± 3.76} \\
    Prm-62 & 0.00 ± 0.00 & 86.95 ± 6.70 & 90.64 ± 5.31 & 90.64 ± 5.52 & \textbf{90.85 ± 5.22} & \textbf{90.85 ± 5.22}\\
    Orc-100 & 0.00 ± 0.00 & 40.76 ± 10.75 & 60.87 ± 11.27 & 60.20 ± 11.12 & \textbf{61.27 ± 11.26} & \textbf{61.27 ± 11.29}\\
    Pdg-18 & 0.00 ± 0.00 & 4.54 ± 3.36 & 28.88 ± 2.15 & 27.09 ± 2.29 & \textbf{30.41 ± 1.90} & \textbf{30.41 ± 1.90}\\
    Pdg-30 & 0.00 ± 0.00 & 2.99 ± 4.79 & 32.37 ± 2.49 & 29.76 ± 2.71 & \textbf{34.66 ± 2.08} & \textbf{34.66 ± 2.08}\\
    Grid20 & 0.00 ± 0.00 & 7.70 ± 2.17 & 7.93 ± 1.56 & 8.16 ± 1.70 & \textbf{8.74 ± 1.83} & 8.54 ± 1.67\\
    Grid40-1 & 0.00 ± 0.00 & 6.65 ± 1.15 & 7.08 ± 1.26 & 8.70 ± 1.22 & \textbf{12.26 ± 1.39} & 10.98 ± 1.05\\
    Grid40-2 & 0.00 ± 0.00 & 8.10 ± 1.07 & 8.45 ± 1.02 & 9.78 ± 1.07 & \textbf{12.48 ± 1.39} & 11.15 ± 0.93 \\
    Grid80 & 0.00 ± 0.00 & -12.26 ± 3.58 & 2.61 ± 0.61 & -12.74 ± 3.95 & -inf & \textbf{39.24 ± 15.37}\\
    DL-5 & 0.00 ± 0.00 & 3.83 ± 17.08 & \textbf{43.40 ± 15.86} & 38.77 ± 13.49 & -inf & 4.34 ± 29.92 \\
    1405 & 0.00 ± 0.00 & -2.19 ± 1.56 & \textbf{2.52 ± 1.98} & 2.39 ± 1.81 & -inf & 1.14 ± 0.91\\
    1407 & 0.00 ± 0.00 & -0.40 ± 0.23 & \textbf{0.33 ± 0.27} & 0.27 ± 0.20 & -inf & 0.23 ± 0.16\\
    Le450-2 & 0.00 ± 0.00 & \textbf{0.19 ± 0.16} & 0.16 ± 0.15 & 0.18 ± 0.17 & -inf & -9.53 ± 2.31\\
    Le450-3 & 0.00 ± 0.00 & 3.18 ± 2.32 & 2.65 ± 2.00 & \textbf{3.25 ± 2.29} & -inf & -9.53 ± 2.31\\
    Le450-4 & 0.00 ± 0.00 & 14.98 ± 8.48 & 14.45 ± 7.95 & \textbf{15.30 ± 8.49} & -inf & -7.46 ± 5.79\\
    Rus20-1 & 0.00 ± 0.00 & 1.44 ± 2.17 & -0.10 ± 3.30 & 1.10 ± 2.32 & 0.61 ± 2.56 &  \textbf{2.20 ± 2.37} \\
    Rus50-1 & 0.00 ± 0.00 & \textbf{2.50 ± 2.91} & -0.02 ± 2.63 & 2.30 ± 2.96 & -5.97 ± 4.06 & -4.65 ± 3.13 \\
    Rus50-2 & 0.00 ± 0.00 & 1.42 ± 2.01 & -0.27 ± 3.02 & \textbf{1.62 ± 2.29} & -6.54 ± 5.02 & -4.75 ± 3.08 \\
    Rus100-1 & 0.00 ± 0.00 & \textbf{2.44 ± 2.17} & 0.80 ± 2.24 & \textbf{2.44 ± 2.17} & -5.64 ± 3.54 & -16.04 ± 5.85 \\
    Rus100-2 & 0.00 ± 0.00 & \textbf{2.12 ± 1.87} & -0.33 ± 0.42 & 2.09 ± 1.87 & -5.74 ± 3.17 & -17.24 ± 5.90\\
    \bottomrule
    \end{tabular}    
    \label{tab:exact-solvers}
\end{table}

Table~\ref{tab:exact-solvers} compares local search methods, including \textsc{BEACON} and baseline methods, against branch-and-bound–based exact solvers. We evaluate two widely used solvers: And/OR Branch and Bound (AOBB)~\citep{AOBB}, using the state-of-the-art \textsc{DAOOPT} implementation~\citep{marinescu2010daoopt}, and \textsc{Toulbar2}~\citep{givry:hal-04021879}. All methods are evaluated on 100 test MPE queries under a uniform time limit of 60 seconds per query. The values reported in each row are computed relative to the \textsc{Greedy} baseline and represent the average percentage gap between the baseline and solver solutions as \(\frac{1}{N}\sum_{i=1}^{N}\frac{\mathcal{LL}_{D}^{(i)}-\mathcal{LL}_{S}^{(i)}}{|\mathcal{LL}_{S}^{(i)}|}\times100\), where \(\mathcal{LL}_{S}^{(i)}\) and \(\mathcal{LL}_{D}^{(i)}\) denote the \textsc{Greedy} and solver log-likelihoods for instance \(i\) after 60 seconds. Positive values indicate improvements over \textsc{Greedy}, while negative values indicate worse performance. Because log-likelihood scales vary across models due to differences in partition functions and potential ranges, rows should not be compared directly.

On smaller networks, including Bayesian Networks (BN), Promedas (Prm), and Pedigree (Pdg), exact solvers achieve superior performance, as expected. In these settings, both GLS-based methods perform competitively, with \textsc{BEACON-GLS+} occasionally outperforming \textsc{GLS+} by a modest margin. In contrast, on larger and more challenging models such as WCSP and RUS, exact solvers degrade substantially relative to local search approaches. In these regimes, \textsc{BEACON}-augmented methods consistently achieve the strongest performance, substantially improving over all baselines. Notably, on the largest instances, \textsc{DAOOPT} fails to return any solution within the time limit, as indicated by \texttt{-inf} in the corresponding rows.

\section{Effect of Weighting Parameter}

The convex combination score used to select neighbors in all our experiments is controlled by a weighting parameter \(\lambda\). The final score is computed as

\[
\mathcal{S}_{\mathrm{final}}(\mathbf{x}') 
= (1-\lambda)\,\tilde{\mathcal{S}}_{\mathrm{LL}}(\mathbf{x}') 
+ \lambda\,\mathcal{S}_{\mathrm{NN}}(\mathbf{x}')
\]

where \(\tilde{\mathcal{S}}_{\mathrm{LL}}(\mathbf{x}')\) denotes the normalized log-likelihood improvement of neighbor \(\mathbf{x}'\) over the current assignment \(\mathbf{x}\), and \(\mathcal{S}_{\mathrm{NN}}(\mathbf{x}')\) represents the neural network’s score for the same neighbor.

\begin{figure}[ht!]
    \centering
    \includegraphics[width=1.0\linewidth]{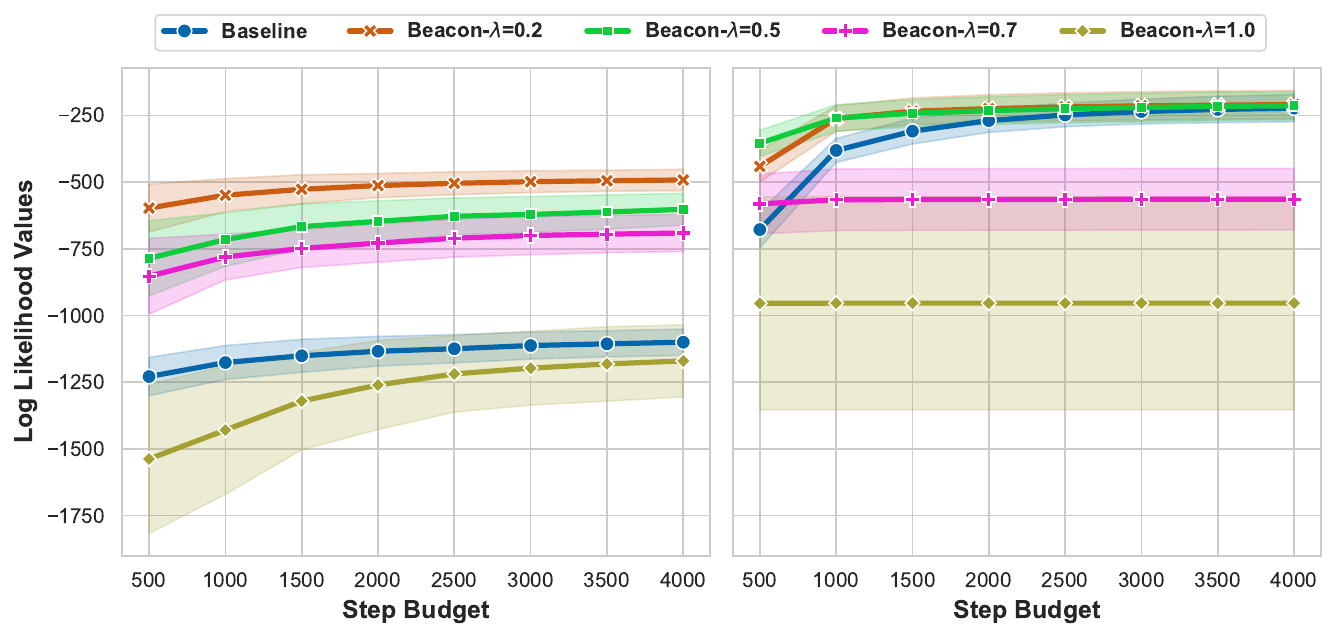}
    \caption{Average log-likelihood scores across 100 MPE queries for different \(\lambda\) values on the BN-30 model. The x-axis shows the step budget, and the y-axis shows the average log-likelihood with standard deviation. The left subfigure corresponds to \textsc{BEACON-Greedy}, and the right subfigure to \textsc{BEACON-GLS+}.}
    \label{fig:lambda-bn30}
\end{figure}

\begin{figure}[ht!]
    \centering
    \includegraphics[width=1.0\linewidth]{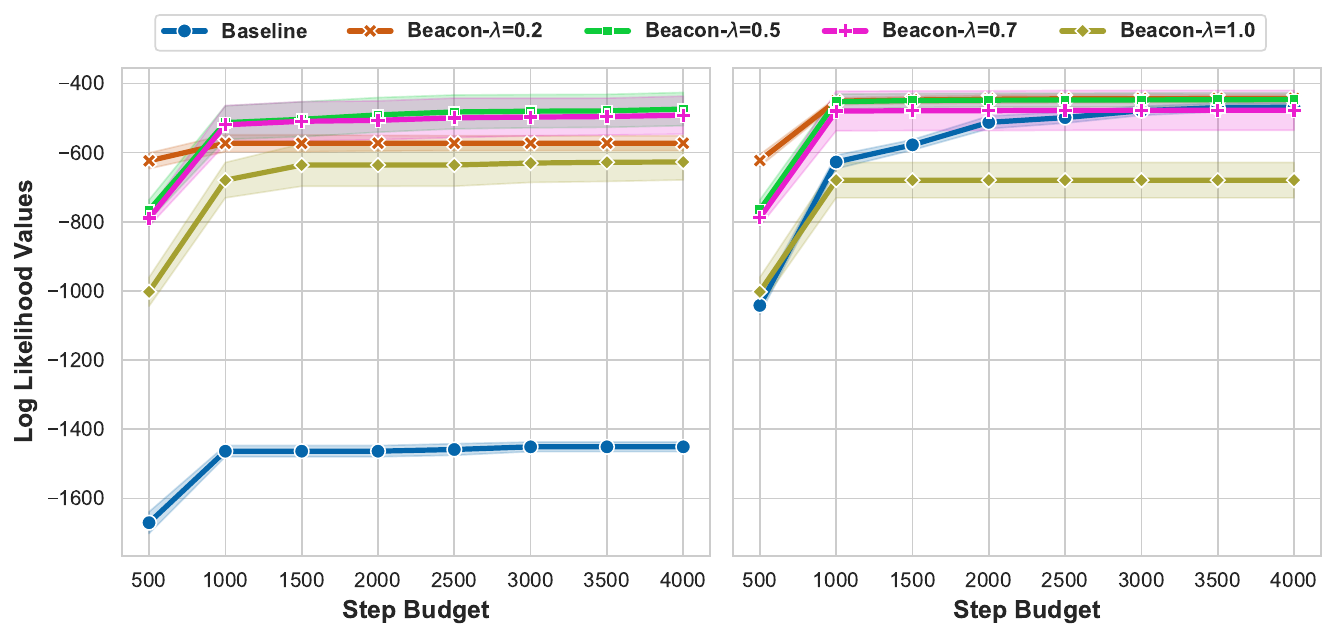}
    \caption{Average log-likelihood scores across 100 MPE queries for different \(\lambda\) values on the BN-32 model. The x-axis shows the step budget, and the y-axis shows the average log-likelihood with standard deviation. The left subfigure corresponds to \textsc{BEACON-Greedy}, and the right subfigure to \textsc{BEACON-GLS+}.}
    \label{fig:lambda-bn32}
\end{figure}

\begin{figure}[ht!]
    \centering
    \includegraphics[width=1.0\linewidth]{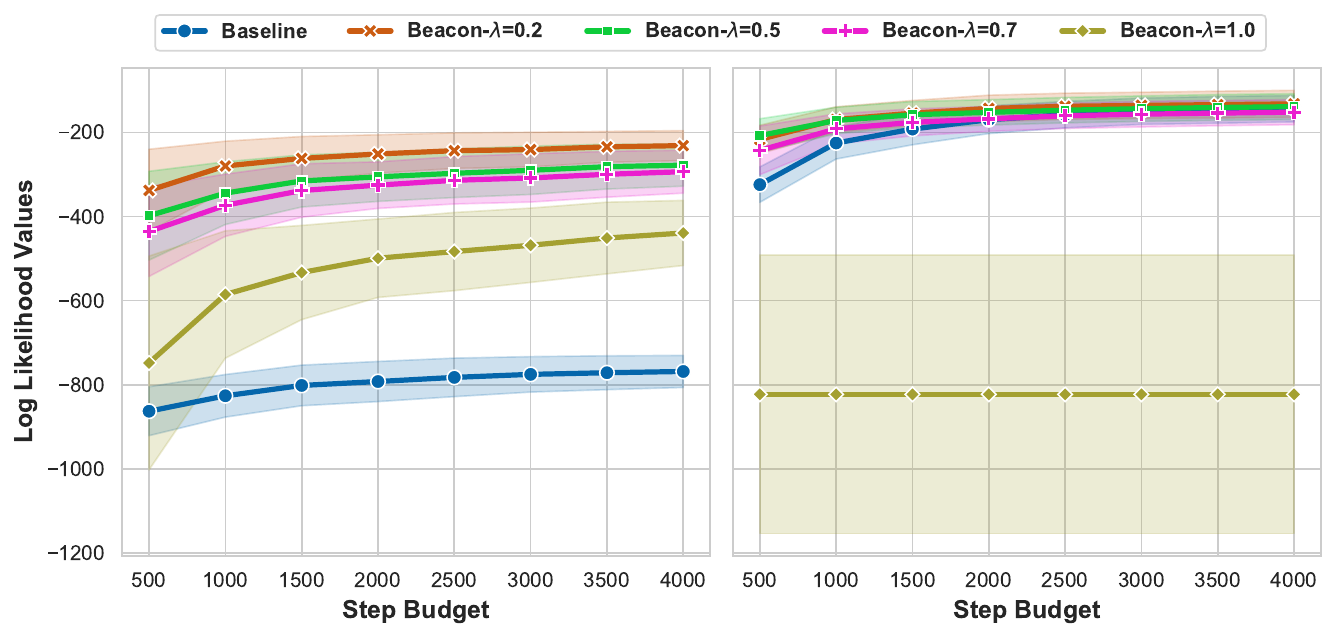}
    \caption{Average log-likelihood scores across 100 MPE queries for different \(\lambda\) values on the BN-45 model. The x-axis shows the step budget, and the y-axis shows the average log-likelihood with standard deviation. The left subfigure corresponds to \textsc{BEACON-Greedy}, and the right subfigure to \textsc{BEACON-GLS+}.}
    \label{fig:lambda-bn45}
\end{figure}

\begin{figure}[ht!]
    \centering
    \includegraphics[width=1.0\linewidth]{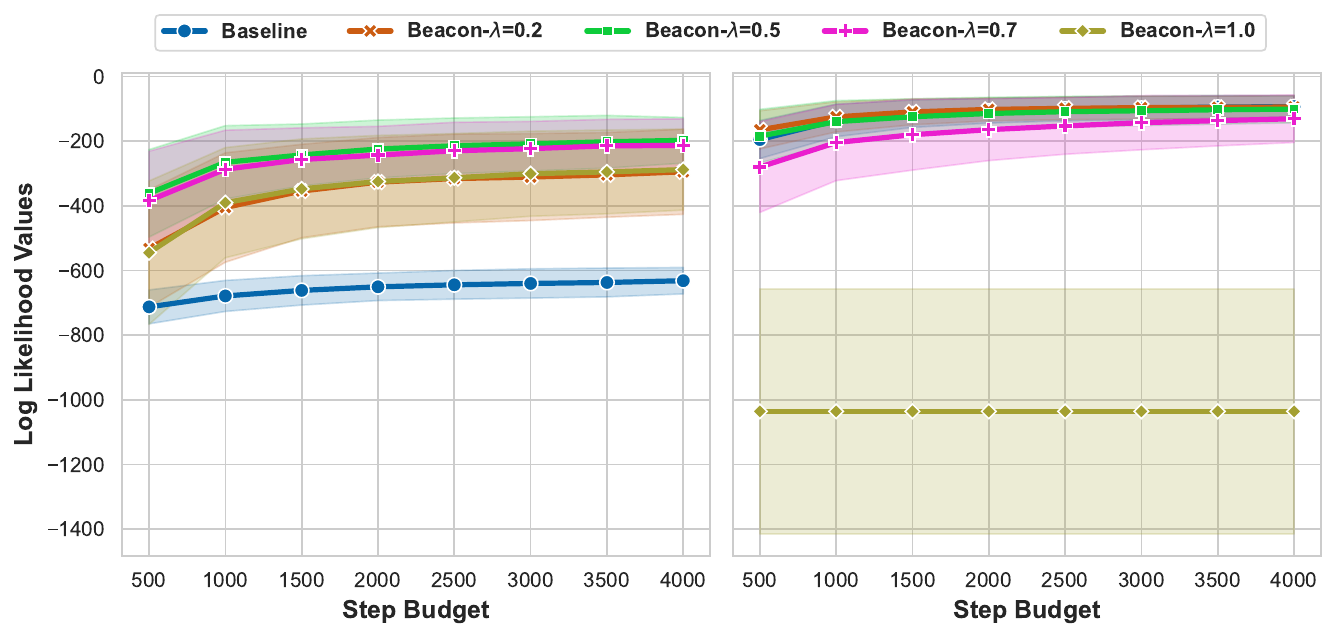}
    \caption{Average log-likelihood scores across 100 MPE queries for different \(\lambda\) values on the BN-61 model. The x-axis shows the step budget, and the y-axis shows the average log-likelihood with standard deviation. The left subfigure corresponds to \textsc{BEACON-Greedy}, and the right subfigure to \textsc{BEACON-GLS+}.}
    \label{fig:lambda-bn61}
\end{figure}

\begin{figure}[ht!]
    \centering
    \includegraphics[width=1.0\linewidth]{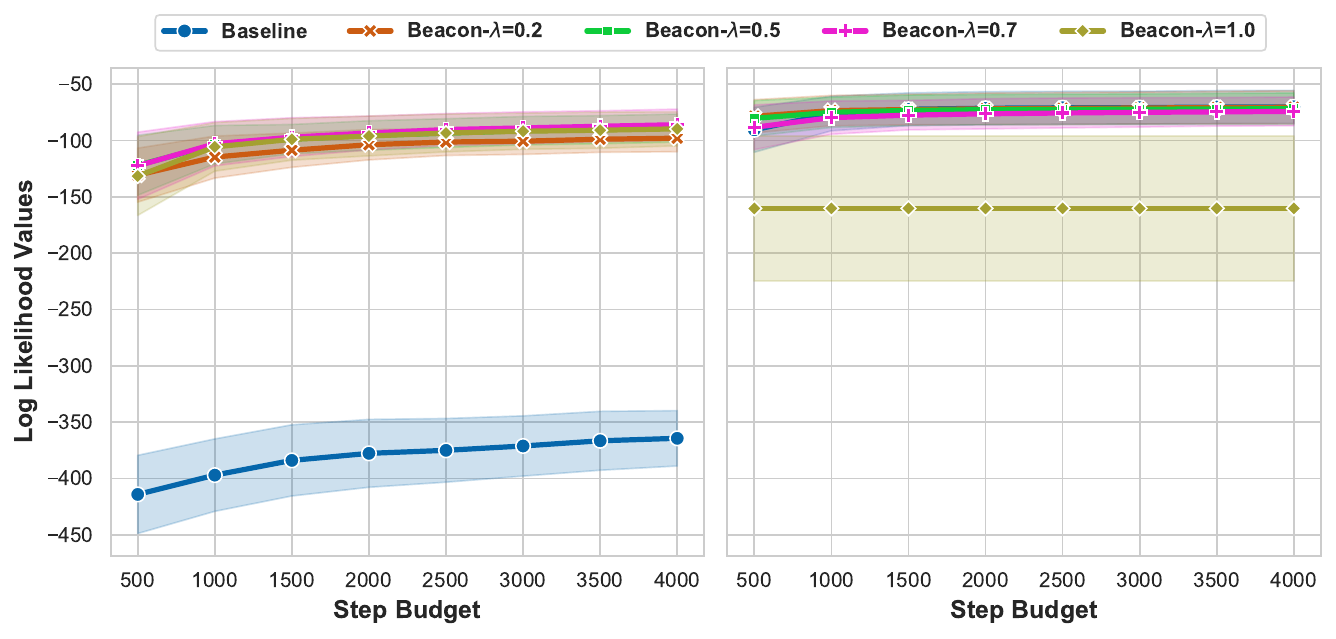}
    \caption{Average log-likelihood scores across 100 MPE queries for different \(\lambda\) values on the BN-65 model. The x-axis shows the step budget, and the y-axis shows the average log-likelihood with standard deviation. The left subfigure corresponds to \textsc{BEACON-Greedy}, and the right subfigure to \textsc{BEACON-GLS+}.}
    \label{fig:lambda-bn65}
\end{figure}

\begin{figure}[ht!]
    \centering
    \includegraphics[width=1.0\linewidth]{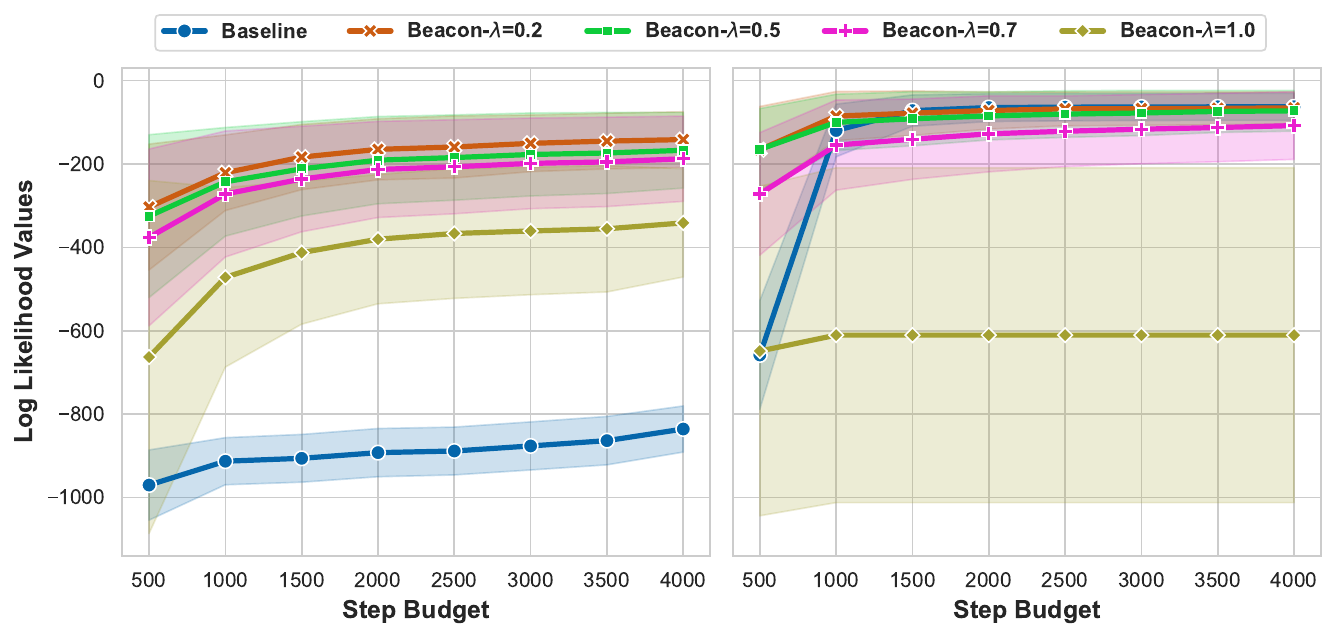}
    \caption{Average log-likelihood scores across 100 MPE queries for different \(\lambda\) values on the Prm60 model. The x-axis shows the step budget, and the y-axis shows the average log-likelihood with standard deviation. The left subfigure corresponds to \textsc{BEACON-Greedy}, and the right subfigure to \textsc{BEACON-GLS+}.}
    \label{fig:lambda-prm60}
\end{figure}

\begin{figure}[ht!]
    \centering
    \includegraphics[width=1.0\linewidth]{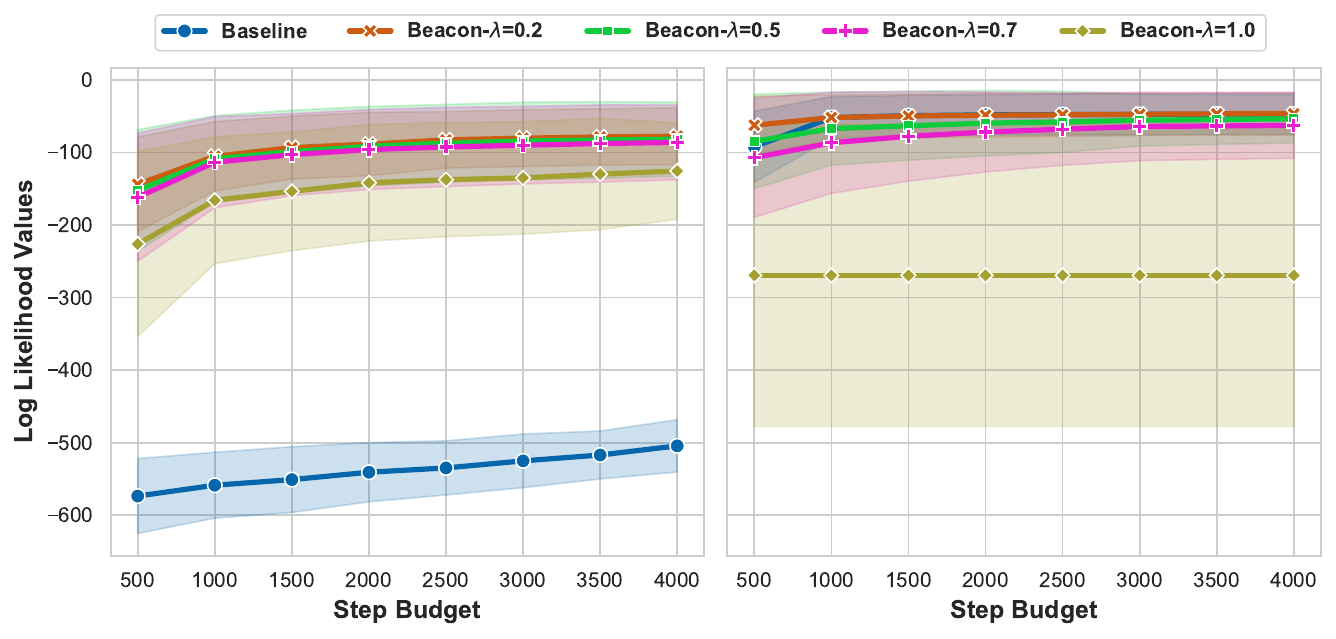}
    \caption{Average log-likelihood scores across 100 MPE queries for different \(\lambda\) values on the Prm62 model. The x-axis shows the step budget, and the y-axis shows the average log-likelihood with standard deviation. The left subfigure corresponds to \textsc{BEACON-Greedy}, and the right subfigure to \textsc{BEACON-GLS+}.}
    \label{fig:lambda-prm62}
\end{figure}

\begin{figure}[ht!]
    \centering
    \includegraphics[width=1.0\linewidth]{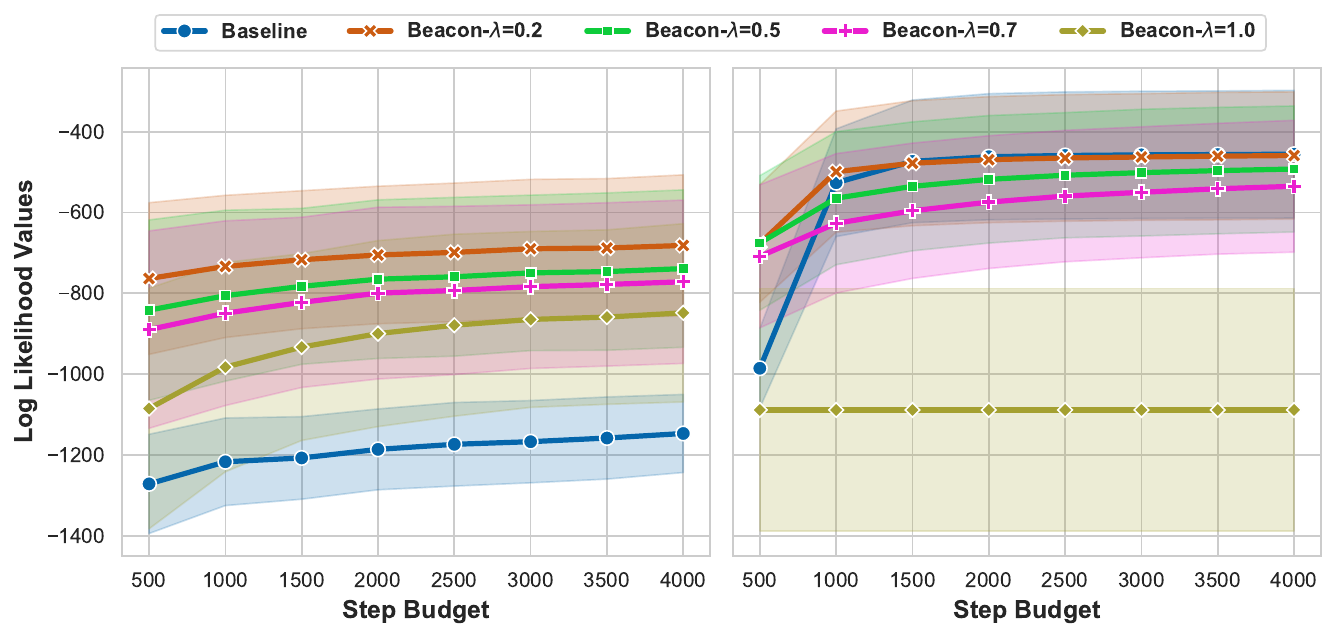}
    \caption{Average log-likelihood scores across 100 MPE queries for different \(\lambda\) values on the Orc-100 model. The x-axis shows the step budget, and the y-axis shows the average log-likelihood with standard deviation. The left subfigure corresponds to \textsc{BEACON-Greedy}, and the right subfigure to \textsc{BEACON-GLS+}.}
    \label{fig:lambda-orc100}
\end{figure}

\begin{figure}[ht!]
    \centering
    \includegraphics[width=1.0\linewidth]{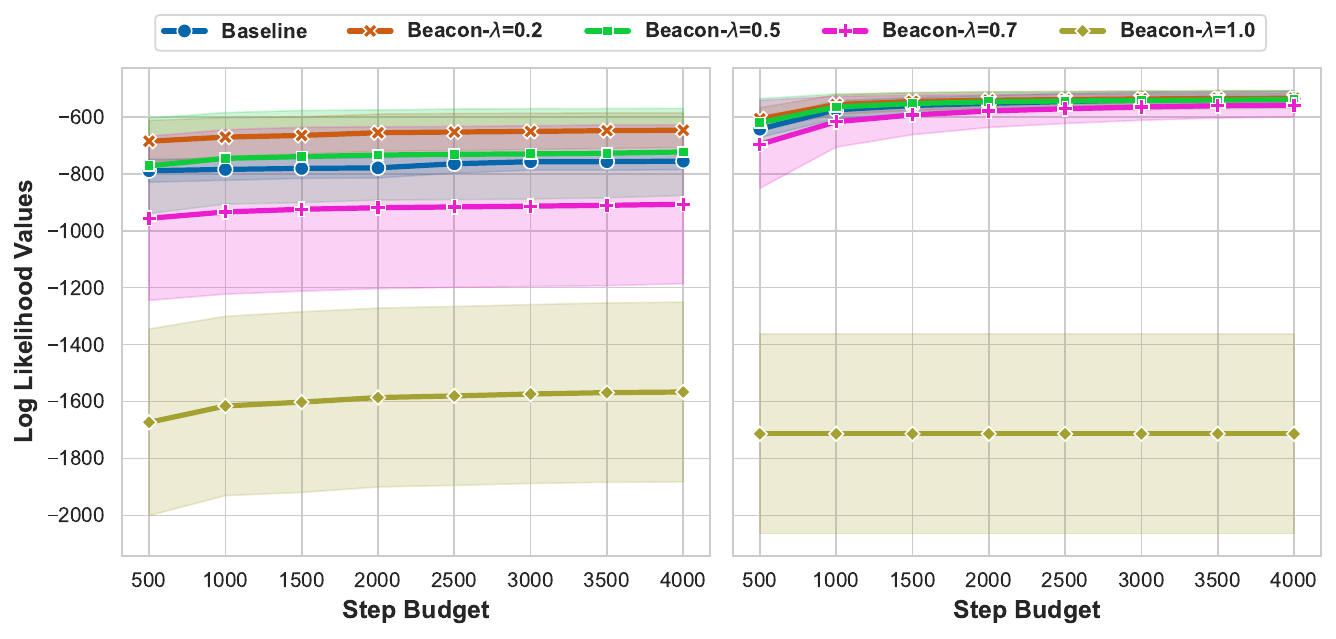}
    \caption{Average log-likelihood scores across 100 MPE queries for different \(\lambda\) values on the pedigree18 model. The x-axis shows the step budget, and the y-axis shows the average log-likelihood with standard deviation. The left subfigure corresponds to \textsc{BEACON-Greedy}, and the right subfigure to \textsc{BEACON-GLS+}.}
    \label{fig:lambda-pedigree18}
\end{figure}

\begin{figure}[ht!]
    \centering
    \includegraphics[width=1.0\linewidth]{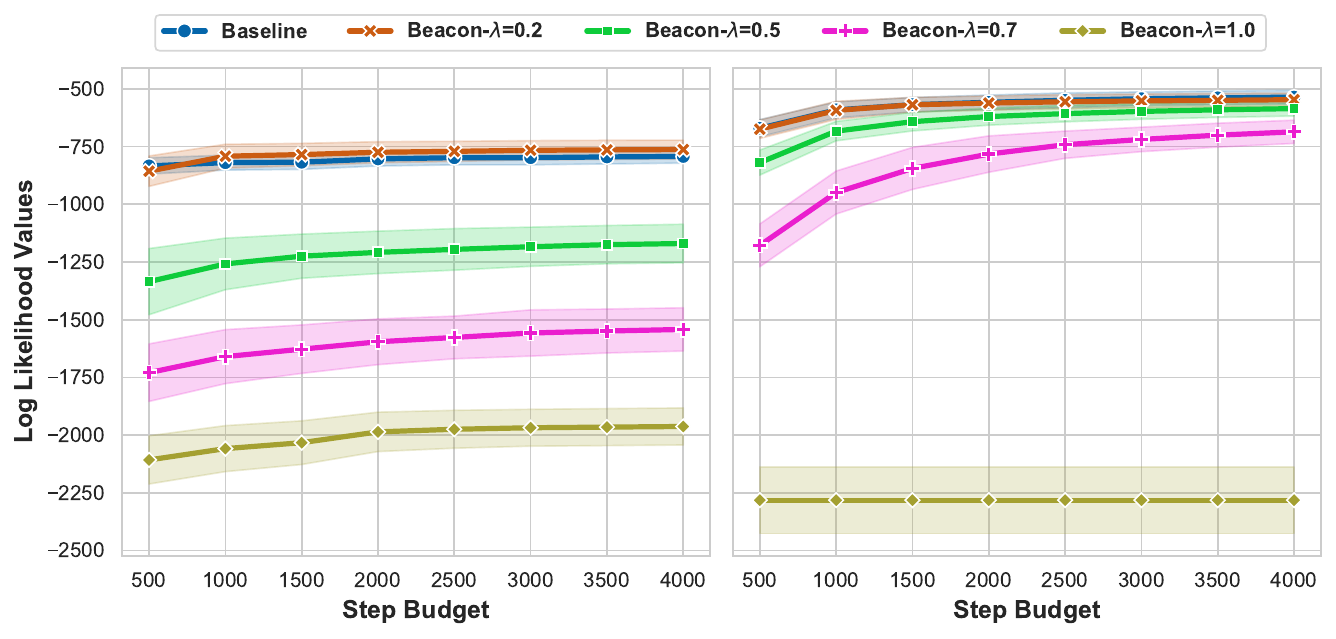}
    \caption{Average log-likelihood scores across 100 MPE queries for different \(\lambda\) values on the pedigree30 model. The x-axis shows the step budget, and the y-axis shows the average log-likelihood with standard deviation. The left subfigure corresponds to \textsc{BEACON-Greedy}, and the right subfigure to \textsc{BEACON-GLS+}.}
    \label{fig:lambda-pedigree30}
\end{figure}

\begin{figure}[ht!]
    \centering
    \includegraphics[width=1.0\linewidth]{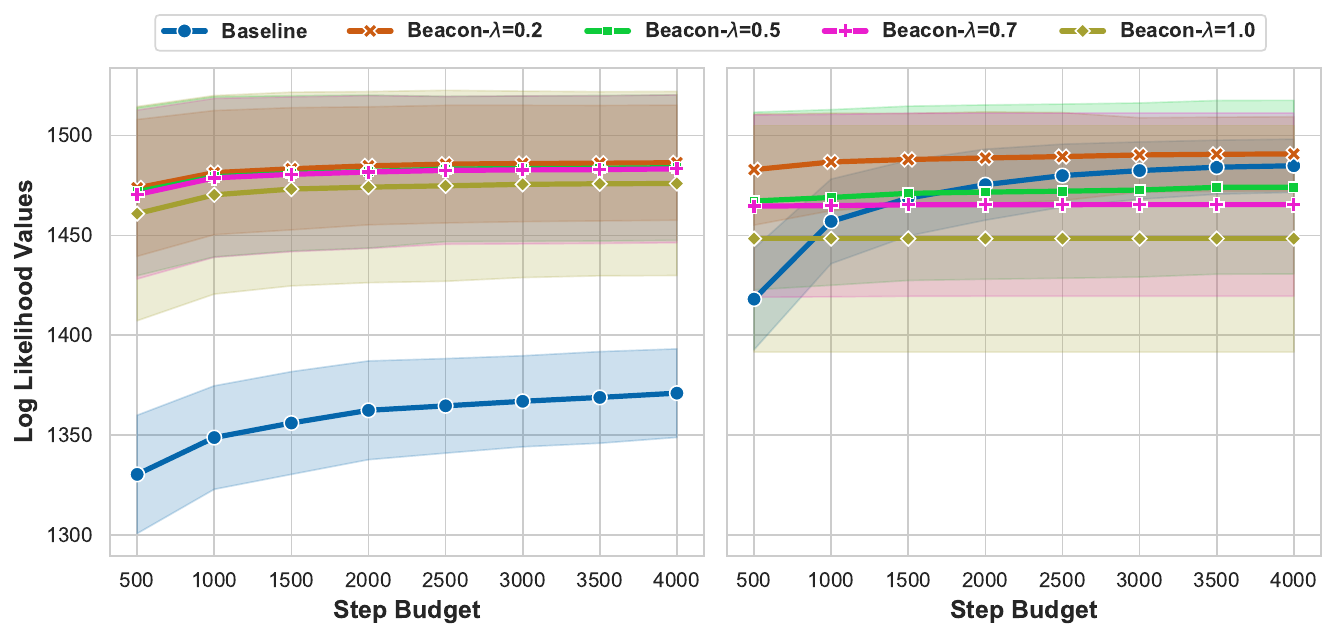}
    \caption{Average log-likelihood scores across 100 MPE queries for different \(\lambda\) values on the Grid20 model. The x-axis shows the step budget, and the y-axis shows the average log-likelihood with standard deviation. The left subfigure corresponds to \textsc{BEACON-Greedy}, and the right subfigure to \textsc{BEACON-GLS+}.}
    \label{fig:lambda-grid20}
\end{figure}

\begin{figure}[ht!]
    \centering
    \includegraphics[width=1.0\linewidth]{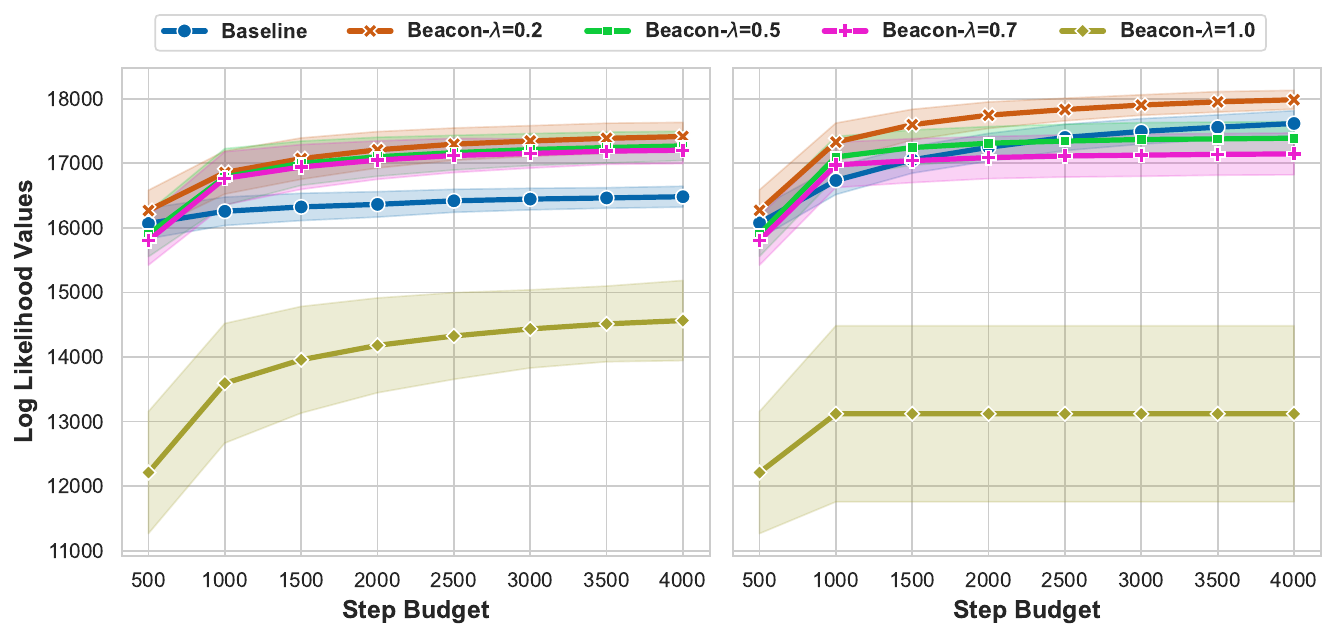}
    \caption{Average log-likelihood scores across 100 MPE queries for different \(\lambda\) values on the Grid40-1 model. The x-axis shows the step budget, and the y-axis shows the average log-likelihood with standard deviation. The left subfigure corresponds to \textsc{BEACON-Greedy}, and the right subfigure to \textsc{BEACON-GLS+}.}
    \label{fig:lambda-grid401}
\end{figure}

\begin{figure}[ht!]
    \centering
    \includegraphics[width=1.0\linewidth]{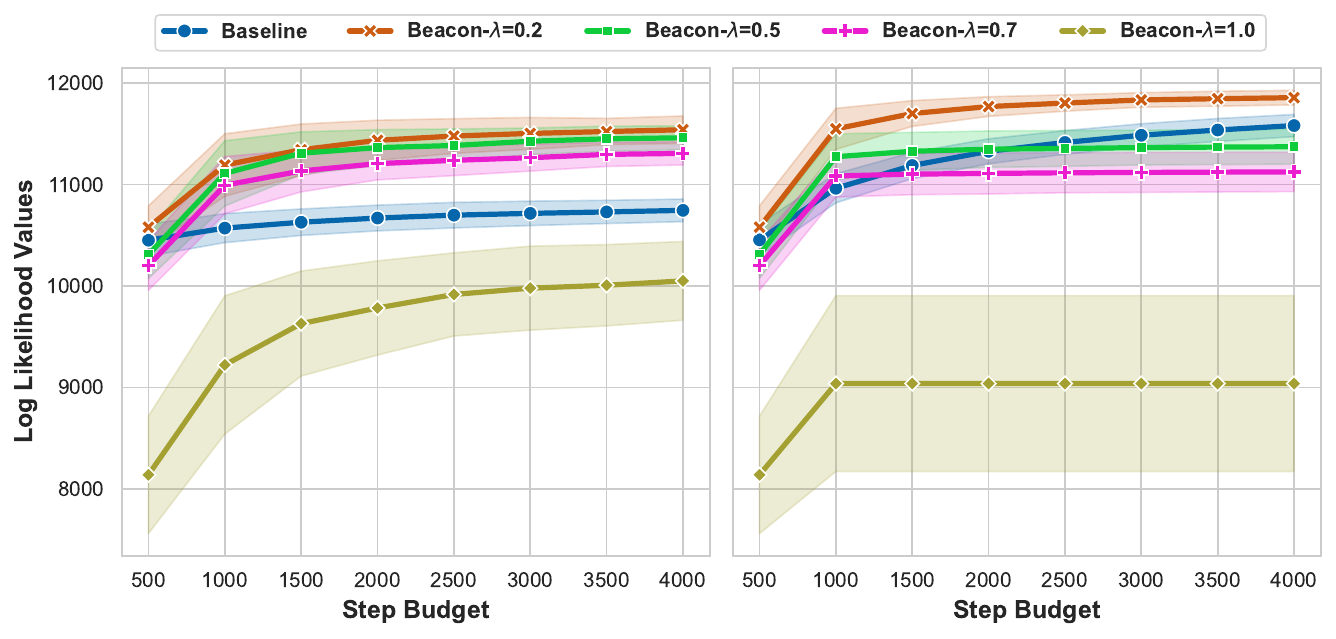}
    \caption{Average log-likelihood scores across 100 MPE queries for different \(\lambda\) values on the Grid40-2 model. The x-axis shows the step budget, and the y-axis shows the average log-likelihood with standard deviation. The left subfigure corresponds to \textsc{BEACON-Greedy}, and the right subfigure to \textsc{BEACON-GLS+}.}
    \label{fig:lambda-grid402}
\end{figure}

\begin{figure}[ht!]
    \centering
    \includegraphics[width=1.0\linewidth]{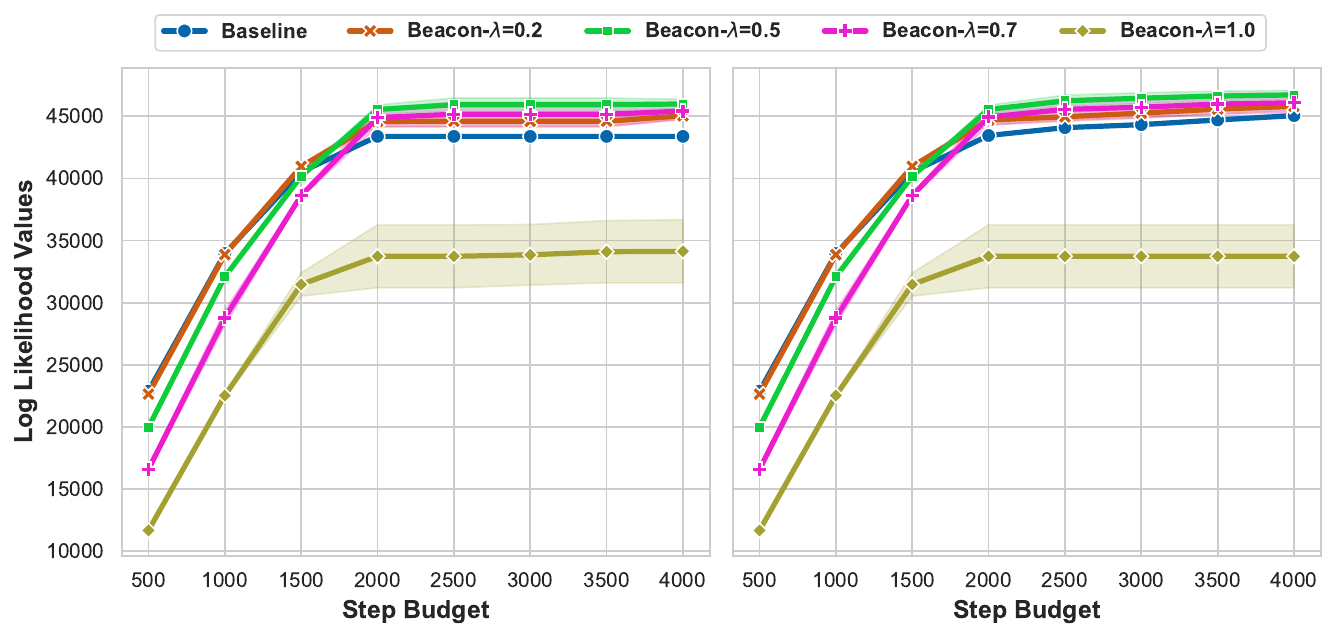}
    \caption{Average log-likelihood scores across 100 MPE queries for different \(\lambda\) values on the Grid80 model. The x-axis shows the step budget, and the y-axis shows the average log-likelihood with standard deviation. The left subfigure corresponds to \textsc{BEACON-Greedy}, and the right subfigure to \textsc{BEACON-GLS+}.}
    \label{fig:lambda-grid80}
\end{figure}

\begin{figure}[ht!]
    \centering
    \includegraphics[width=1.0\linewidth]{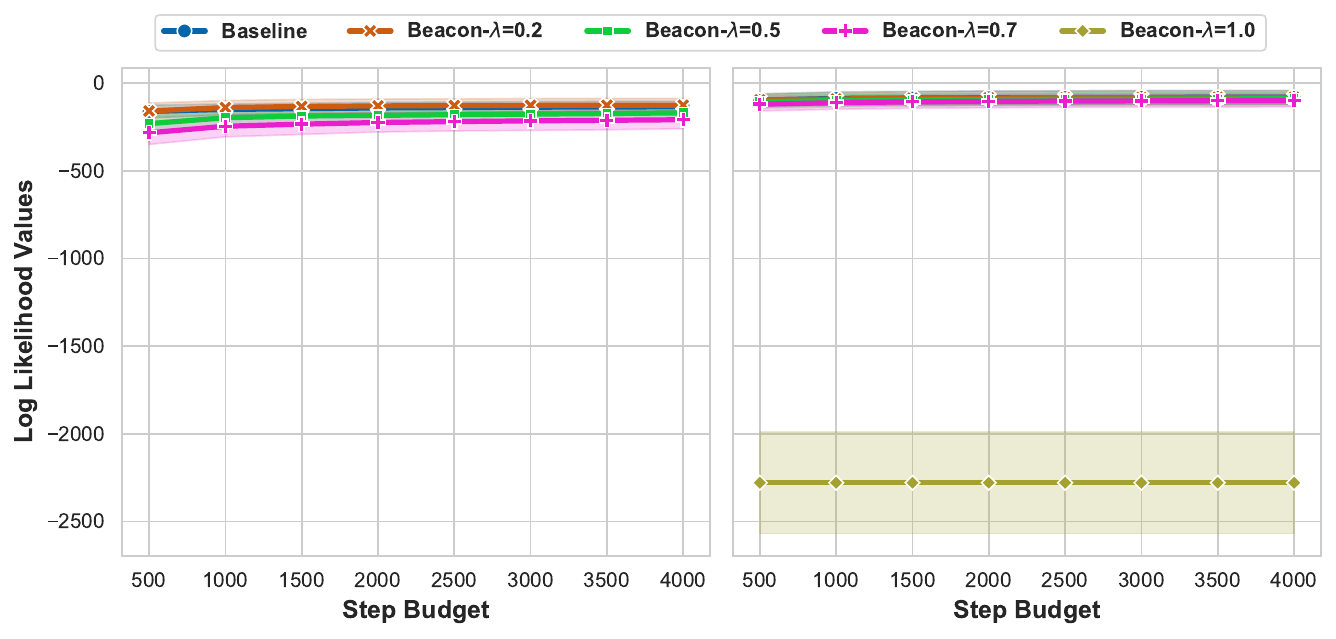}
    \caption{Average log-likelihood scores across 100 MPE queries for different \(\lambda\) values on the DL-5 model. The x-axis shows the step budget, and the y-axis shows the average log-likelihood with standard deviation. The left subfigure corresponds to \textsc{BEACON-Greedy}, and the right subfigure to \textsc{BEACON-GLS+}.}
    \label{fig:lambda-dl5}
\end{figure}

\begin{figure}[ht!]
    \centering
    \includegraphics[width=1.0\linewidth]{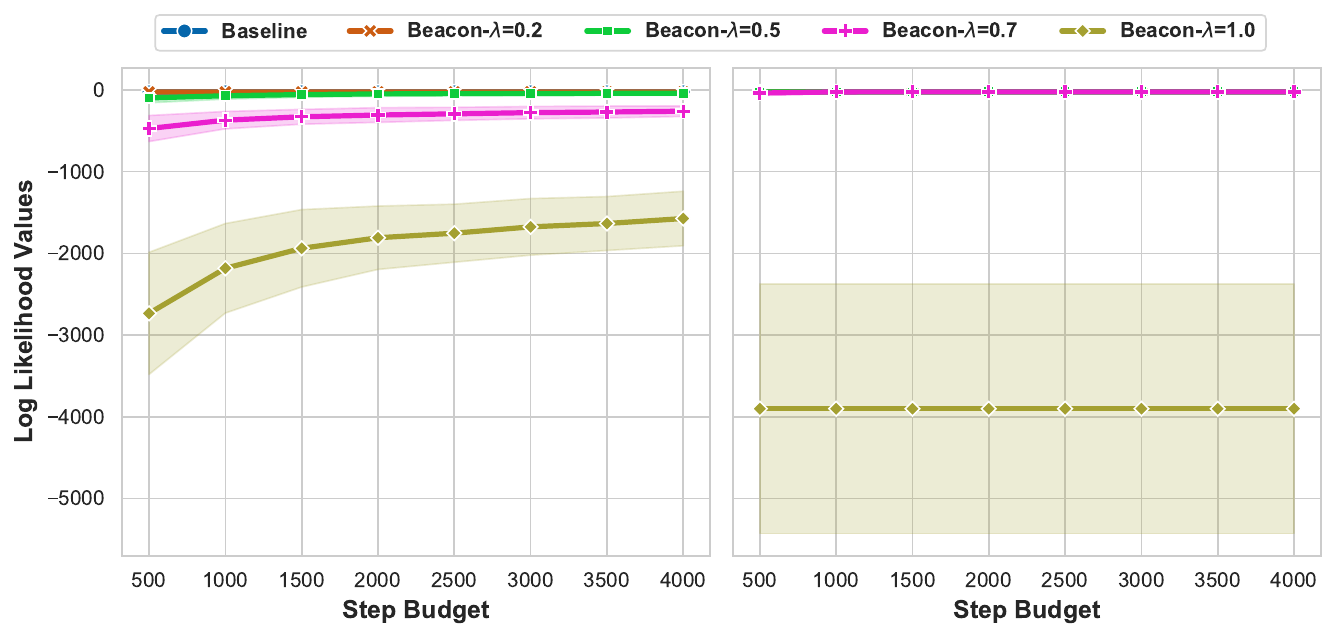}
    \caption{Average log-likelihood scores across 100 MPE queries for different \(\lambda\) values on the 1405 model. The x-axis shows the step budget, and the y-axis shows the average log-likelihood with standard deviation. The left subfigure corresponds to \textsc{BEACON-Greedy}, and the right subfigure to \textsc{BEACON-GLS+}.}
    \label{fig:lambda-1405}
\end{figure}

\begin{figure}[ht!]
    \centering
    \includegraphics[width=1.0\linewidth]{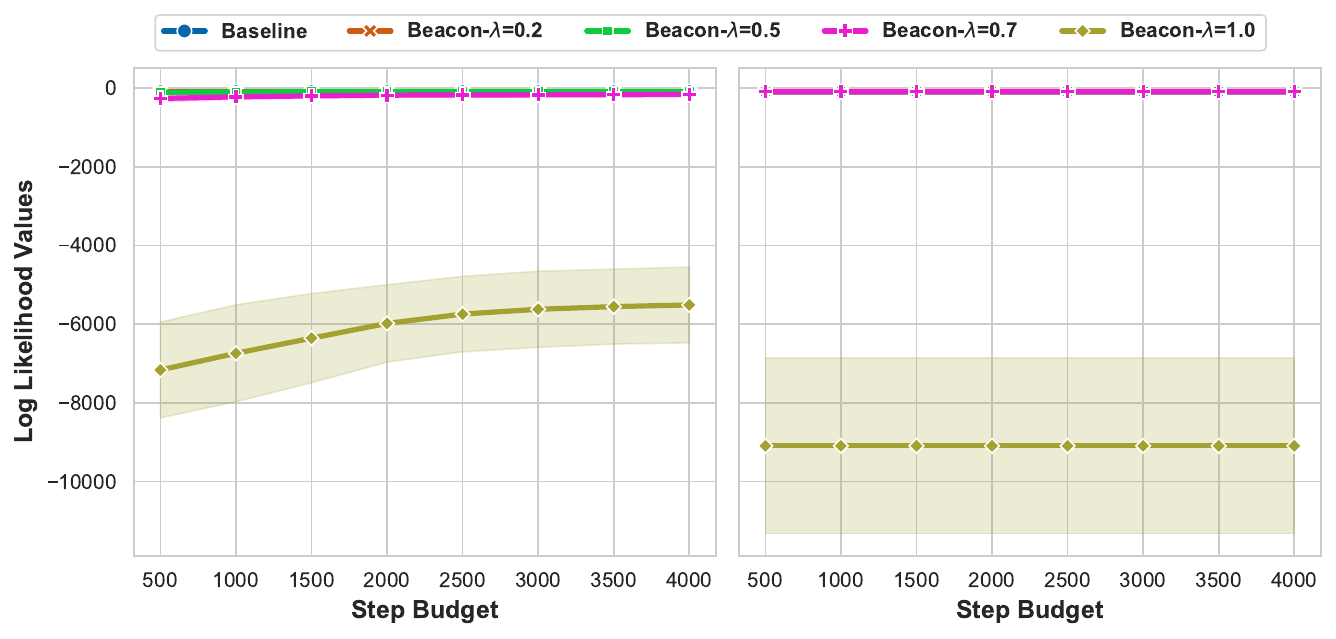}
    \caption{Average log-likelihood scores across 100 MPE queries for different \(\lambda\) values on the 1407 model. The x-axis shows the step budget, and the y-axis shows the average log-likelihood with standard deviation. The left subfigure corresponds to \textsc{BEACON-Greedy}, and the right subfigure to \textsc{BEACON-GLS+}.}
    \label{fig:lambda-1407}
\end{figure}

\begin{figure}[ht!]
    \centering
    \includegraphics[width=1.0\linewidth]{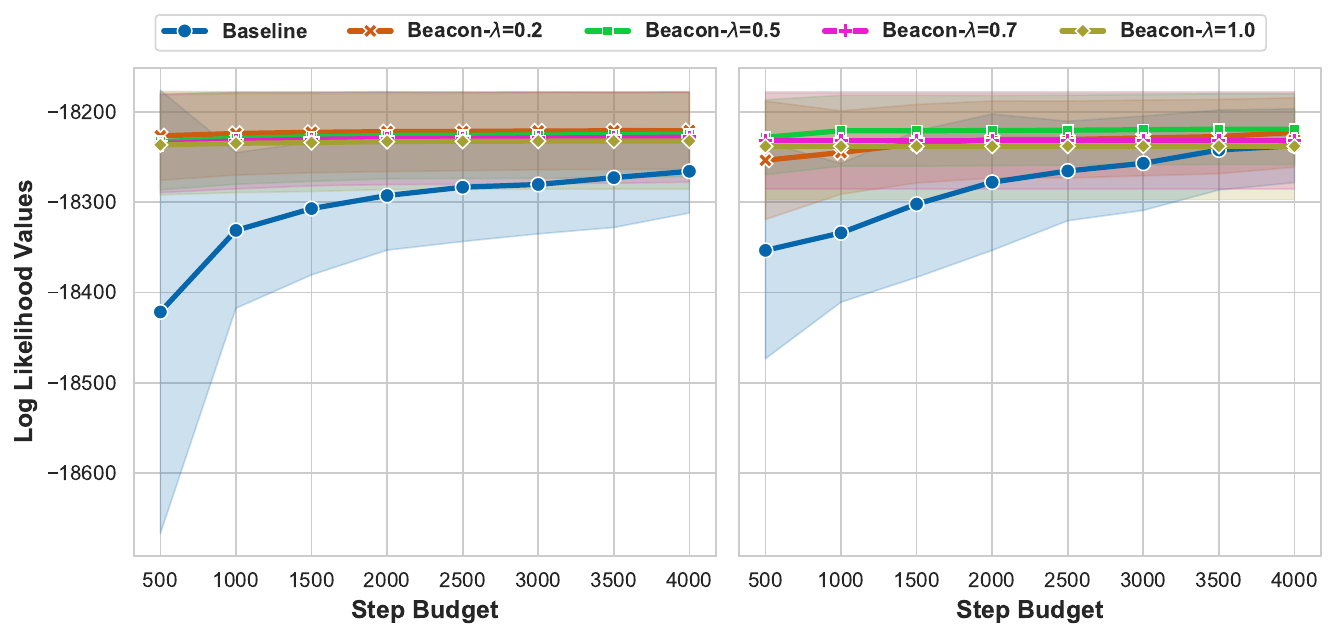}
    \caption{Average log-likelihood scores across 100 MPE queries for different \(\lambda\) values on the Le450-2 model. The x-axis shows the step budget, and the y-axis shows the average log-likelihood with standard deviation. The left subfigure corresponds to \textsc{BEACON-Greedy}, and the right subfigure to \textsc{BEACON-GLS+}.}
    \label{fig:lambda-le2}
\end{figure}

\begin{figure}[ht!]
    \centering
    \includegraphics[width=1.0\linewidth]{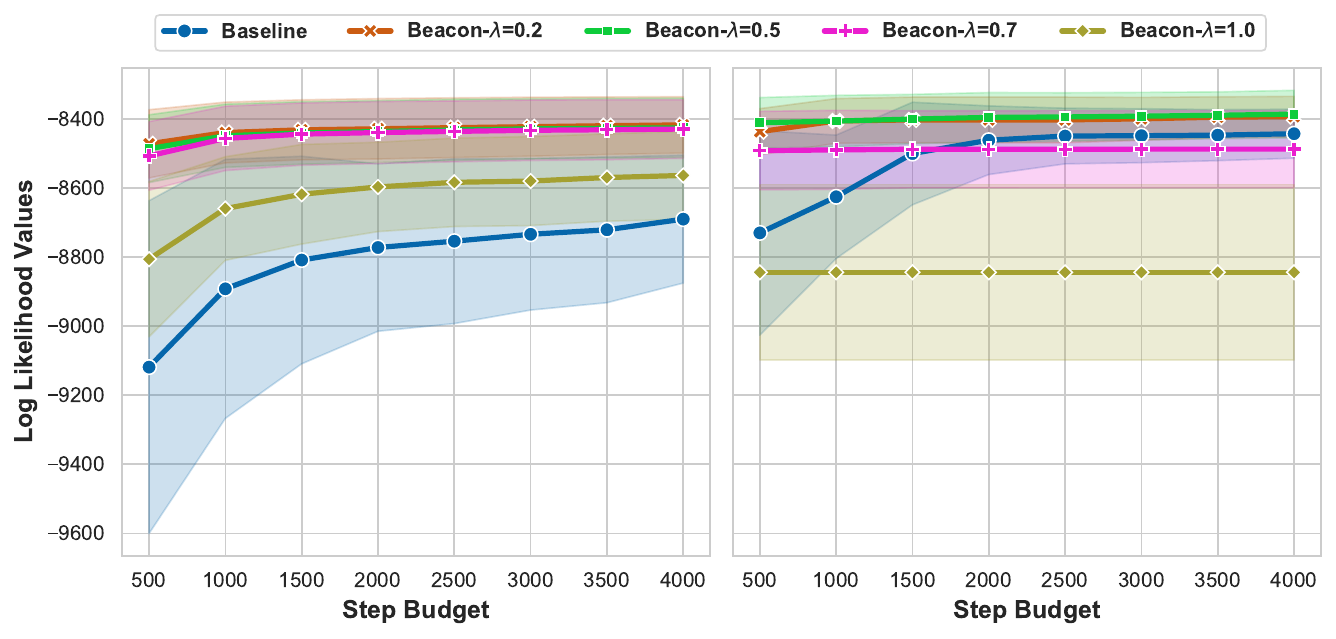}
    \caption{Average log-likelihood scores across 100 MPE queries for different \(\lambda\) values on the Le450-3 model. The x-axis shows the step budget, and the y-axis shows the average log-likelihood with standard deviation. The left subfigure corresponds to \textsc{BEACON-Greedy}, and the right subfigure to \textsc{BEACON-GLS+}.}
    \label{fig:lambda-le3}
\end{figure}

\begin{figure}[ht!]
    \centering
    \includegraphics[width=1.0\linewidth]{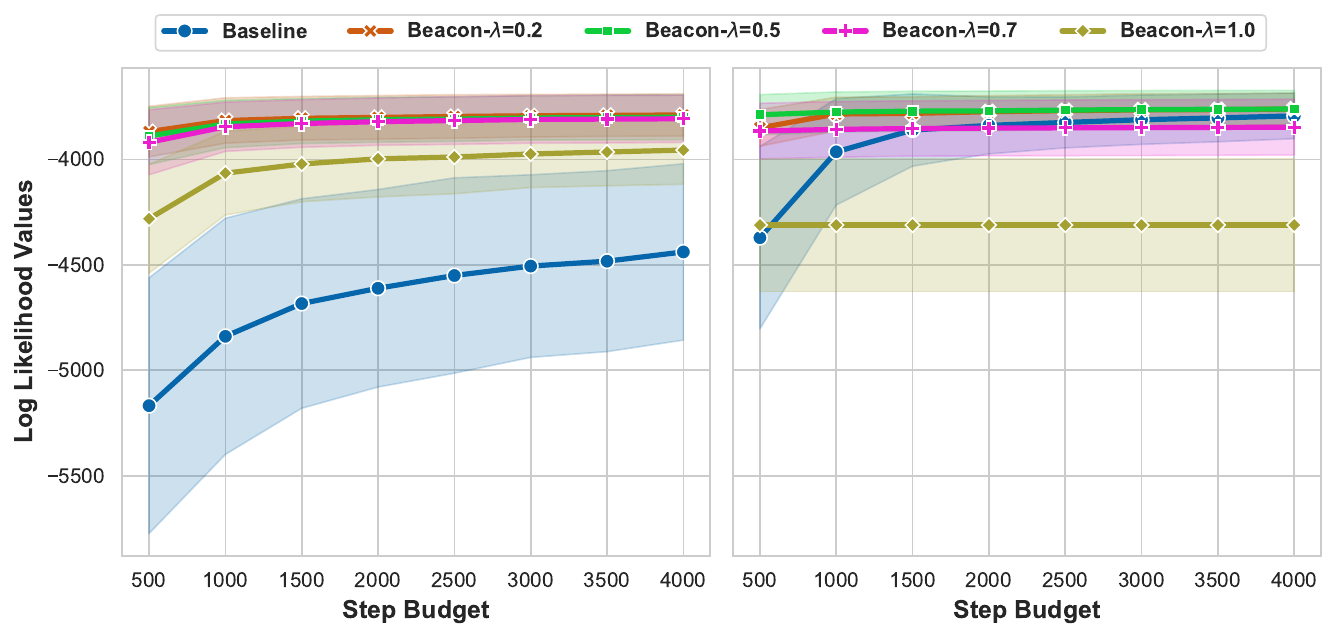}
    \caption{Average log-likelihood scores across 100 MPE queries for different \(\lambda\) values on the Le450-4 model. The x-axis shows the step budget, and the y-axis shows the average log-likelihood with standard deviation. The left subfigure corresponds to \textsc{BEACON-Greedy}, and the right subfigure to \textsc{BEACON-GLS+}.}
    \label{fig:lambda-le4}
\end{figure}

\begin{figure}[ht!]
    \centering
    \includegraphics[width=1.0\linewidth]{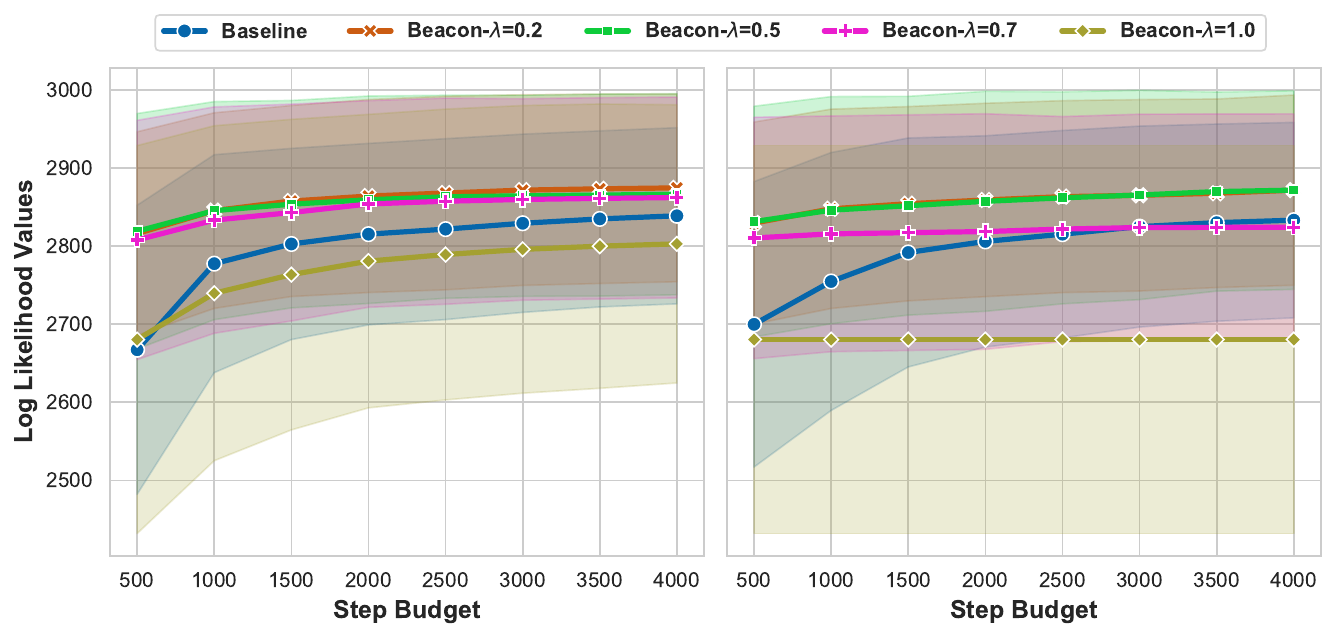}
    \caption{Average log-likelihood scores across 100 MPE queries for different \(\lambda\) values on the Rus20-1 model. The x-axis shows the step budget, and the y-axis shows the average log-likelihood with standard deviation. The left subfigure corresponds to \textsc{BEACON-Greedy}, and the right subfigure to \textsc{BEACON-GLS+}.}
    \label{fig:lambda-rus201}
\end{figure}

\begin{figure}[ht!]
    \centering
    \includegraphics[width=1.0\linewidth]{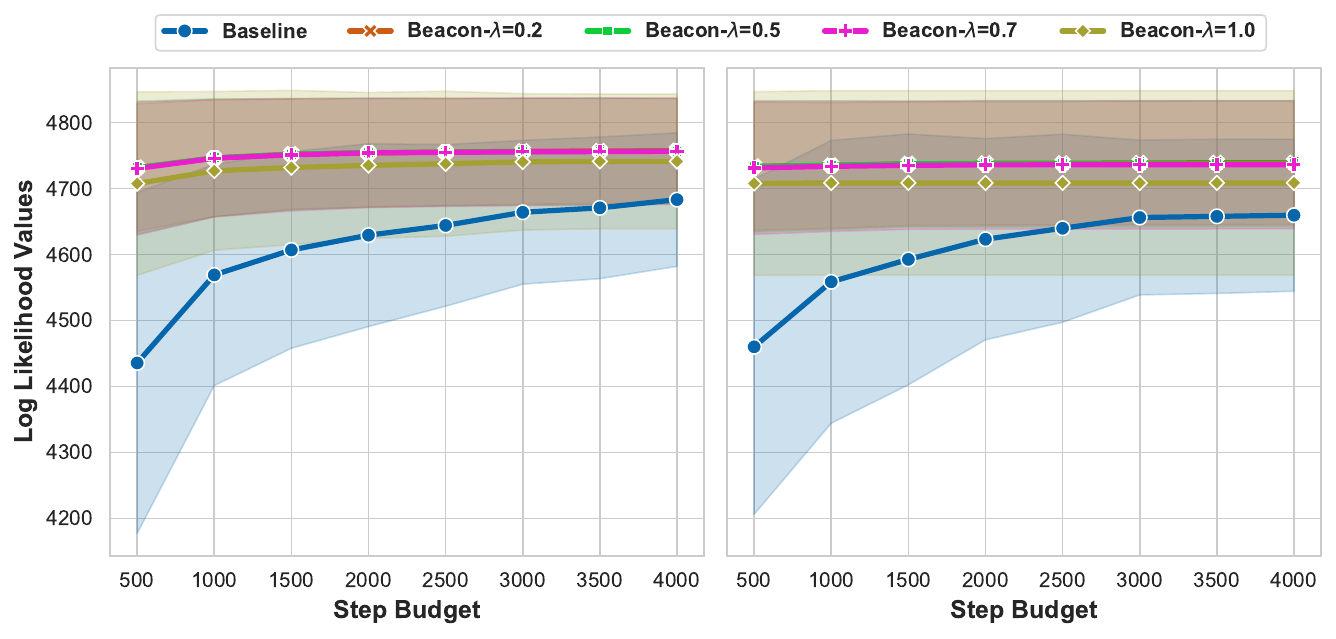}
    \caption{Average log-likelihood scores across 100 MPE queries for different \(\lambda\) values on the Rus50-1 model. The x-axis shows the step budget, and the y-axis shows the average log-likelihood with standard deviation. The left subfigure corresponds to \textsc{BEACON-Greedy}, and the right subfigure to \textsc{BEACON-GLS+}.}
    \label{fig:lambda-rus501}
\end{figure}

\begin{figure}[ht!]
    \centering
    \includegraphics[width=1.0\linewidth]{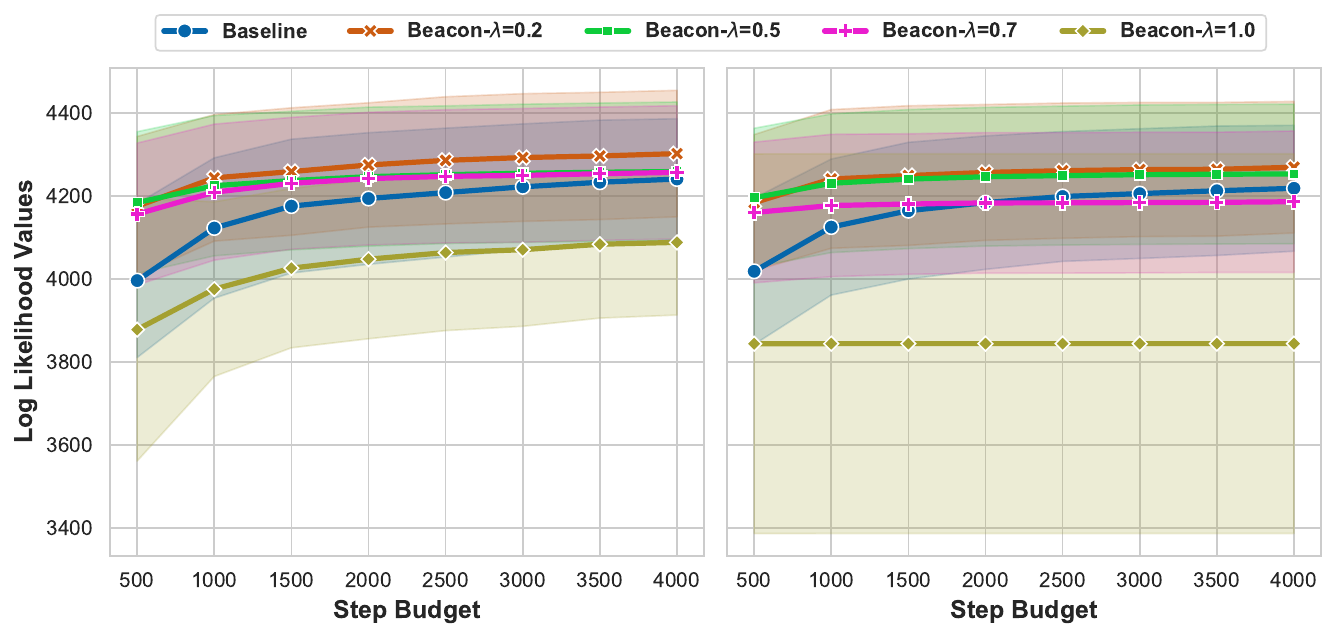}
    \caption{Average log-likelihood scores across 100 MPE queries for different \(\lambda\) values on the Rus50-2 model. The x-axis shows the step budget, and the y-axis shows the average log-likelihood with standard deviation. The left subfigure corresponds to \textsc{BEACON-Greedy}, and the right subfigure to \textsc{BEACON-GLS+}.}
    \label{fig:lambda-rus502}
\end{figure}

\begin{figure}[ht!]
    \centering
    \includegraphics[width=1.0\linewidth]{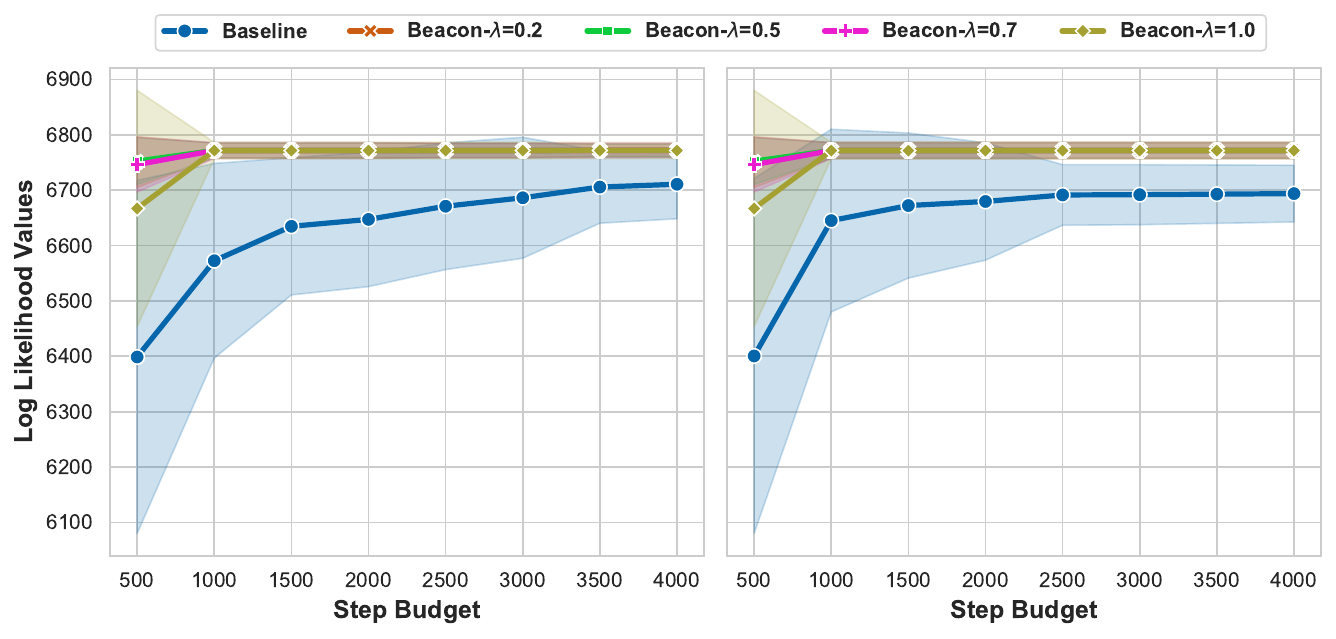}
    \caption{Average log-likelihood scores across 100 MPE queries for different \(\lambda\) values on the Rus100-1 model. The x-axis shows the step budget, and the y-axis shows the average log-likelihood with standard deviation. The left subfigure corresponds to \textsc{BEACON-Greedy}, and the right subfigure to \textsc{BEACON-GLS+}.}
    \label{fig:lambda-rus1001}
\end{figure}

\begin{figure}[ht!]
    \centering
    \includegraphics[width=1.0\linewidth]{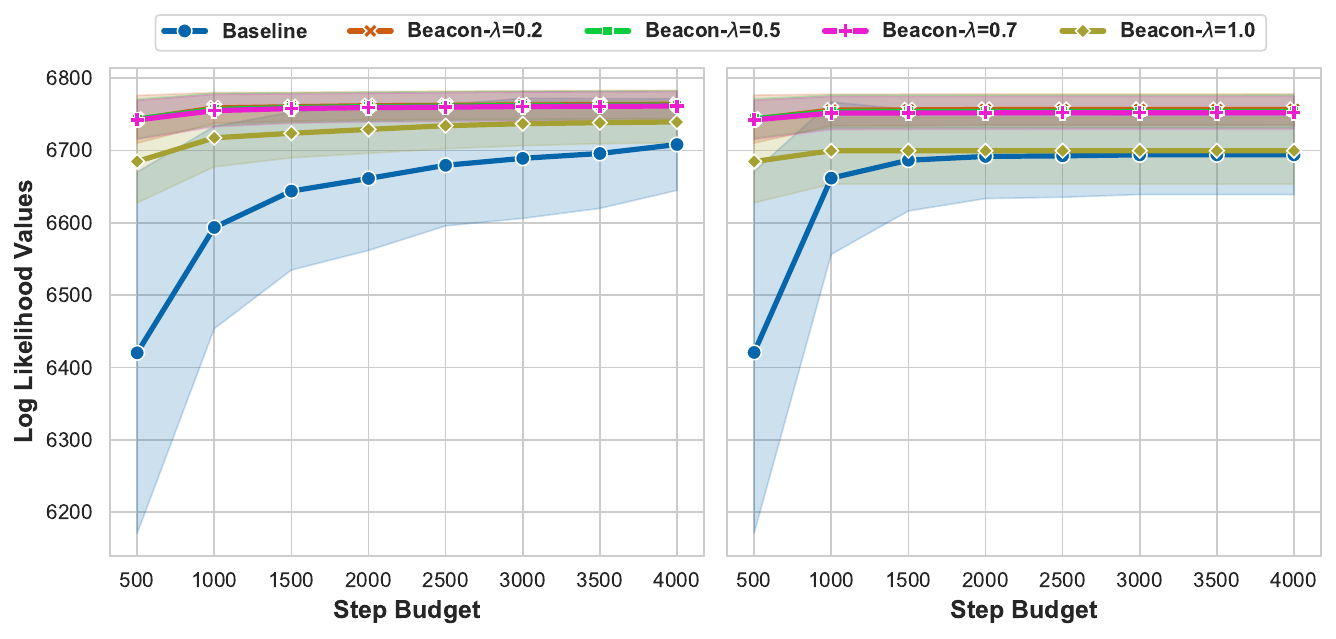}
    \caption{Average log-likelihood scores across 100 MPE queries for different \(\lambda\) values on the Rus100-2 model. The x-axis shows the step budget, and the y-axis shows the average log-likelihood with standard deviation. The left subfigure corresponds to \textsc{BEACON-Greedy}, and the right subfigure to \textsc{BEACON-GLS+}.}
    \label{fig:lambda-rus1002}
\end{figure}

\FloatBarrier

The choice of \(\lambda\) directly influences the performance of \textsc{BEACON}. Figures~\ref{fig:lambda-bn30}–\ref{fig:lambda-rus1002} illustrate this effect for \textsc{BEACON-Greedy} and \textsc{BEACON-GLS+}. Each figure reports the average log-likelihood over a validation set of 100 MPE queries for four values of \(\lambda = \{0.2, 0.5, 0.7, 1.0\}\). In all plots, the x-axis shows the step budget (ranging from 500 to 4000 in increments of 500), and the y-axis shows the average log-likelihood with standard deviation. The left subfigure presents \textsc{BEACON-Greedy} for varying $\lambda$ alongside the \textsc{Greedy} baseline, while the right subfigure presents \textsc{BEACON-GLS+} alongside the \textsc{GLS+} baseline.

Across most models, incorporating \textsc{BEACON} improves performance over the baselines, with smaller values of \(\lambda\) (0.2–0.5) consistently yielding the strongest gains. This suggests that combining modest neural guidance with likelihood-based search provides the best trade-off between exploration and exploitation. Larger values \(\lambda\) of can occasionally degrade performance, although in several cases \(\lambda\)=0.7 and even \(\lambda\) = 1.0 still outperform the corresponding baseline.

Interestingly, sensitivity to \(\lambda\) differs between the two solvers. \textsc{BEACON-Greedy}, which relies entirely on local scoring for exploration, exhibits larger performance variation and benefits most from neural guidance. In contrast, \textsc{BEACON-GLS+} exhibits more stable but comparatively smaller gains, reflecting the already strong exploratory behavior of GLS+, while still delivering consistent improvements over the baseline.

\end{document}